\documentclass[10pt]{article}
\usepackage{lipsum}
\usepackage{calc}
\usepackage[
  letterpaper,%
  textheight={47\baselineskip+\topskip},%
  textwidth={\paperwidth-126pt},%
  footskip=75pt,%
  marginparwidth=0pt,%
  top={\topskip+0.75in}]{geometry}
\usepackage{fancyhdr}

\usepackage{times}
\usepackage{soul}
\usepackage{url}
\usepackage[hidelinks]{hyperref}
\usepackage[utf8]{inputenc}
\usepackage[small]{caption}
\usepackage{graphicx}
\usepackage{amsmath}
\usepackage{amsthm}
\usepackage{booktabs}
\usepackage{amsfonts}
\usepackage{enumitem}
\usepackage{comment}
\usepackage{xcolor}
\usepackage{tabularx,booktabs}
\usepackage{multirow}
\usepackage{caption}
\usepackage{subcaption}
\usepackage[ruled]{algorithm2e}

\usepackage{mathtools}

\newtheorem{theorem}{Theorem}

\newtheorem{definition}[theorem]{Definition}
\newtheorem{proposition}[theorem]{Proposition}

\DeclareMathOperator{\Tr}{tr}

\DeclareMathOperator{\cov}{cov}
\DeclareMathOperator{\maximize}{maximize}
\DeclareMathOperator{\minimize}{minimize}

\newcommand*\samethanks[1][\value{footnote}]{\footnotemark[#1]}

\title{Learning to Transfer with von Neumann Conditional Divergence\\
(Supplementary Material)
}

\title{Learning to Transfer with von Neumann Conditional Divergence}
\author{Ammar Shaker\textsuperscript{\rm 1}\thanks{A. Shaker and S. Yu are the corresponding authors} , Shujian Yu\textsuperscript{\rm 2,3}{\samethanks} , and Daniel O{\~n}oro-Rubio\textsuperscript{\rm 1}\\[3mm]
\textsuperscript{\rm 1} NEC Laboratories Europe, Heidelberg, Germany\\
\textsuperscript{\rm 2} UiT - The Arctic University of Norway, Tromsø, Norway\\
\textsuperscript{\rm 3} Xi'an Jiaotong University, Xi'an, Shaanxi, China\\
ammar.shaker@neclab.eu, yusj9011@gmail.com, daniel.onoro@neclab.eu}

\makeatletter
\def\@maketitle{%
  \vbox to 2.25in{%
    \hsize\textwidth
    \linewidth\hsize
    \vfil
    \centering
    {\LARGE \@title \par}
    \vskip 2em
    {\large \begin{tabular}[t]{c}\@author \end{tabular}\par}
    \vfil}}
\makeatother
\begin{document}
\maketitle 

\begin{abstract}
The similarity of feature representations plays a pivotal role in the success of problems related to domain adaptation. Feature similarity includes both the invariance of marginal distributions and the closeness of conditional distributions given the desired response $y$ (e.g., class labels). Unfortunately, traditional methods always learn such features without fully taking into consideration the information in $y$, which in turn may lead to a mismatch of the conditional distributions or the mix-up of discriminative structures underlying data distributions. In this work, we introduce the recently proposed von Neumann conditional divergence to improve the transferability across multiple domains. 
We show that this new divergence is differentiable and eligible to easily quantify the functional dependence between features and $y$. Given multiple source tasks, we integrate this divergence to capture discriminative information in $y$ and design novel learning objectives assuming those source tasks are observed either simultaneously or sequentially. In both scenarios, we obtain favorable performance against state-of-the-art methods in terms of smaller generalization error on new tasks and less catastrophic forgetting on source tasks (in the sequential setup).
\end{abstract}

\section{Introduction}

Deep learning has achieved remarkable successes in diverse machine learning problems and applications~\cite{pouyanfar2018survey}. However, most of deep learning applications are limited to a single or isolated task, in which a network is usually trained from scratch based on a large scale labeled dataset~\cite{donahue2014decaf}. As a result, the training of deep neural networks becomes frustrating when labeled data is scarce or expensive to obtain. In these scenarios, the efficient transfer of information from one or multiple tasks to another and the prevention of negative transfer amongst all tasks become fundamental techniques for the successful deployment of a deep learning system~\cite{yosinski2014transferable,riemer_learning_2018}. 

Different problems arise depending on the number of tasks and how tasks arrive (e.g., concurrently or sequentially). These problems range from the standard domain adaptation from a single source domain to a target domain~\cite{pan2010domain}, up to the continual learning which trains a single network on a series of interrelated tasks~\cite{parisi2019continual,de2019continual}, with the goal of improving positive transfer and mitigating negative interference~\cite{riemer_learning_2018}.


Tremendous efforts have been made to improve transferability across multiple domains~\cite{ganin2016domain,zhao2018adversarial,zhao2019learning}. Most of the works aim to learn domain-invariant features $\mathbf{t}$ without the knowledge of class label or desired response $y$. Common techniques to match feature marginal distributions include the maximum mean discrepancy (MMD)~\cite{pan2010domain,zhu2019aligning}, the moment matching~\cite{zellinger2017central}, 
the $\mathcal{H}$ divergence~\cite{zhao2018adversarial}, the Wasserstein distance~\cite{wang2019tmda}, etc.
For classification, $p(y|\mathbf{t})$ can be modeled with a multinomial distribution~\cite{pei2018multi,zhao2020domain}. However, it is still an open problem to explicitly capture the functional dependence between $\mathbf{t}$ and $y$ for regression.    

Let us consider a network that consists of a feature extractor $f_\theta:\mathcal{X}\rightarrow \mathcal{T}$ (parametrized by $\theta$) and a predictor $h_\varphi:\mathcal{T}\rightarrow \mathcal{Y}$ (parameterized by $\varphi$); the similarity of latent representation $\mathbf{t}$ includes two aspects: the invariance of marginal distributions (i.e., $p(f_\theta(\mathbf{x}))$) across different domains and the functional closeness of using $\mathbf{t}$ to predict $y$. The predictive power of $h_\varphi$ can be characterized by the conditional distribution $p(y|\mathbf{t})$. From an information-theoretic perspective, the conditional entropy $H(y|\mathbf{t})=-\mathbb{E}(\log(p(y|\mathbf{x})))$ also measures the dependence between $y$ and $\mathbf{t}$. 

Our main contributions are summarized as follows:
\begin{itemize}
  \item{We introduce the von Neumann conditional divergence $D_{vN}$~\cite{yu2020measuring} to the problems of domain adaptation. This new divergence can easily quantify the functional dependence between latent features $\mathbf{t}$ and the desired response $y$, in both classification and regression.}
  \item{We show the utility of $D_{vN}$ in a standard domain adaptation setup in which multiple source tasks are observed either simultaneously (\textit{a.k.a.}, multi-source domain adaptation) or sequentially (\textit{a.k.a.}, continual learning).}
  \item{For multi-source domain adaptation (MSDA),}
    \begin{itemize}
      \item{Given a hypothesis set $\mathcal{H}$ and the new loss function induced by $D_{vN}$, we define a new domain discrepancy distance $\mathcal{D}_{\text{M-disc}}(P,Q)$ to measure the closeness of two distributions $P$ and $Q$.}
      \item{By generating a weighted source domain $D_\alpha$ with probability $P_\alpha=\sum_{i=1}^{K}{w_i P_{s_i}}$ (subject to $\sum_{i=1}^{K}{w_i=1}$), in which $P_{s_i}$ denotes the distribution of the $i$-th source domain, we derive a new generalization bound based on $\mathcal{D}_{\text{M-disc}}$ for MSDA.}
      \item{We design a new objective based on the derived bound and optimize it as a min-max game. Compared to four state-of-the-art (SOTA) methods, our approach reduces the generalization error and identifies meaningful strength of “relatedness” from each source to the target domain.}
    \end{itemize}
  \item{For the problem of continual learning (CL),}
    \begin{itemize}
        \item{We show that the functional similarity of latent features $\mathbf{t}$ to the desired response $y$ is able to quantify the importance of network parameters to previous tasks. Based on this observation, we develop a new regularization-based CL approach by network modularization~\cite{watanabe2018modular}.}
        \item{We compare our approach with the baseline elastic weight consolidation (EWC)~\cite{kirkpatrick2017overcoming} and three other SOTA methods on five benchmark datasets. Empirical results demonstrate that our approach reduces catastrophic forgetting and is less sensitive to the choice of hyper-parameters.}
    \end{itemize}
    
\end{itemize}

\section{Background Knowledge}
\label{Background_Knowledge}

\subsection{Problem Setup}
Let $\mathcal{X}$ and $\mathcal{Y}$ be the input and the desired response (e.g., class labels) spaces. Given $K$ source domains (or tasks) $\{D_i\}_{i=1}^K$, we obtain $N_i$ training samples $\{\mathbf{x}_i^j,y_i^j\}_{j=1}^{N_i}$ in the $i$-th source $D_i$, which follows a distribution $P_i(\mathbf{x},y)$ (defined over $\mathcal{X}\times\mathcal{Y})$.

In a typical (unsupervised) domain adaptation setup, the goal is to generalize a parametric model learned from data samples in $\{D_i\}_{i=1}^K$ to a different, but related, target domain $D_{K+1}$ following a new distribution $P_{K+1}(\mathbf{x},y)$, in which we assume no access to the true response $y$ in the data sampled from $P_{K+1}(\mathbf{x},y)$, i.e., minimizing the objective
\begin{equation}
    \mathbb{E}_{(\mathbf{x},y)\sim D_{K+1}}\left[\ell(w;\mathbf{x},y)\right],
\end{equation}
where $\ell(w;\mathbf{x},y):\mathcal{W}\rightarrow\mathbb{R}$ is the loss function of $w$ associated with sample $(\mathbf{x},y)$, and $\mathcal{W}\subseteq\mathbb{R}^d$ is the model parameter space.

In an online scenario where tasks arrive sequentially, lifelong learning searches for models minimizing the population loss over all seen $(K+1)$ tasks, where access to previous tasks $\{D_i\}_{i=1}^K$ is either limited or prohibited:
\begin{equation}
    \sum_{i=1}^{K+1}{\mathbb{E}_{(\mathbf{x},y)\sim D_i}\left[\ell(w;\mathbf{x},y)\right]}.
\end{equation}

Obviously, this poses new challenges, as the network is required to ensure positive transfer from $\{D_i\}_{i=1}^K$ to $D_{K+1}$, and, at the same time, avoid negative interference to its performance on $\{D_i\}_{i=1}^K$.

In this work, we consider multi-source domain adaptation for regression (i.e., $y\in\mathbb{R}$) and a standard continual learning setup on image classification (i.e., $y$ contains $m$ unique categories $\{c_1,\dots,c_{m}\}$).

\subsection{von Neumann Conditional Divergence}
Let us draw $N$ samples from two joint distributions $P_1(\mathbf{x},y)$ and $P_2(\mathbf{x},y)$, i.e., $\{\mathbf{x}_1^i,y_1^i\}_{i=1}^{N}$ and $\{\mathbf{x}_2^i,y_2^i\}_{i=1}^{N}$. Here, $y$ refers to the response variable, and $\mathbf{x}$ can be either the raw input variable or the feature vector $\mathbf{z}=f_\theta(\mathbf{x})$ after a feature extractor $f_\theta:\mathcal{X}\rightarrow\mathcal{Z}$ parameterized by $\theta$.

Yu et al. \cite{yu2020measuring} define the relative divergence from $P_1\left(y|\mathbf{x}\right)$ to $P_2\left(y|\mathbf{x}\right)$ as:
\begin{equation}\label{eq:VN_conditional_divergence}
D(P_1(y|\mathbf{x})\|P_2(y|\mathbf{x})) = D_{vN}(\sigma_{\mathbf{x}y}\|\rho_{\mathbf{x}y}) - D_{vN}(\sigma_{\mathbf{x}}\|\rho_{\mathbf{x}}), 
\end{equation}
where $\sigma_{\mathbf{x}y}$ and $\rho_{\mathbf{x}y}$ denote the sample covariance matrices evaluated on $\{\mathbf{x}_1^i,y_1^i\}_{i=1}^{N}$ and $\{\mathbf{x}_2^i,y_2^i\}_{i=1}^{N}$, respectively. 
Similarly, $\sigma_{x}$ and $\rho_{x}$ refer to the sample covariance matrices evaluated on $\{\mathbf{x}_1^i\}_{i=1}^{N}$ and $\{\mathbf{x}_2^i\}_{i=1}^{N}$, respectively. 
$D_{vN}$ is the von Neumann divergence~\cite{nielsen2002quantum,kulis2009low}, $D_{vN}(\sigma\|\rho) = \Tr(\sigma \log \sigma - \sigma \log \rho - \sigma + \rho)$, which operates on two symmetric positive definite (SPD) matrices, $\sigma$ and $\rho$. Eq.~(\ref{eq:VN_conditional_divergence}) is not symmetric. To achieve symmetry, one can simply take the form:
\begin{equation}\label{eq:VN_conditional_divergence_sym}
\begin{aligned}
D(P_1(y|\mathbf{x}):P_2(y|\mathbf{x}))  &=& \frac{1}{2}\large(D(P_1(y|\mathbf{x})\|P_2(y|\mathbf{x})) + D(P_2(y|\mathbf{x})\|P_1(y|\mathbf{x}))\large).
\end{aligned}
\end{equation}



{\color{black}As a complement to~\cite{yu2020measuring}, we additionally provide the convergence behavior analysis of the matrix-based von Neumann divergence on sample covariance matrix to the true distributional distance (see supplementary material), although this is not the main contribution of this work.

Note that, aligning distributions or conditional distributions always plays a pivotal role in different domain adaptation related problems. Before our work, the MMD has been extensively investigated. However, there is no universal agreement on the definition of conditional MMD~\cite{park2020measure}, and most of existing operator-based approaches on conditional MMD depend on stringent assumptions which are usually violated in practice (e.g.,~\cite{ren2016conditional}). This unfortunate fact urges the need for exploring the possibility of a new divergence measure that is both simple to compute and differentiable. Moreover, compared to MMD that relies on a kernel function with width $\sigma$ which is always hard to tune in practice, Eqs.~(\ref{eq:VN_conditional_divergence}) and (\ref{eq:VN_conditional_divergence_sym}) defined over sample covariance matrix are hyper-parameter free.   
}

\section{Interpreting the von Neumann Conditional Divergence as a Loss Function}
In case $P_1(\mathbf{x},y)$ and $P_2(\mathbf{x},y)$ have the same marginal distribution $P(\mathbf{x})$ or share the same input variable $\mathbf{x}$ (i.e., $\sigma_\mathbf{x}=\rho_\mathbf{x}$), the symmetric von Neumann conditional divergence (Eq.~(\ref{eq:VN_conditional_divergence_sym})) reduces to:
{\fontsize{9.pt}{10.8pt} \selectfont 
\begin{equation}\label{eq:Jeffery_VN}
    D(P_1(y|\mathbf{x}):P_2(y|\mathbf{x}))=\frac{1}{2}
    \Tr {\left(\left(\sigma_{\mathbf{x}y}-\rho_{\mathbf{x}y}\right)\left(\log{\sigma_{\mathbf{x}y}}-\log{\rho_{\mathbf{x}y}}\right)\right)}.
\end{equation}
}
We term the r.h.s. of Eq.~(\ref{eq:Jeffery_VN}) as the Jeffery von Neumann divergence on $\sigma_{\mathbf{x}y}$ and $\rho_{\mathbf{x}y}$, and denote it as $J_{vN}(\sigma_{\mathbf{x}y}:\rho_{\mathbf{x}y})$. 

Taking $X=\sigma_{\mathbf{x},f(\mathbf{x})}$ and $Y=\sigma_{\mathbf{x},\hat{f}(\mathbf{x})}$, $\sqrt{J_{vN}(\sigma_{\mathbf{x},f(\mathbf{x})}:\sigma_{\mathbf{x},\hat{f}(\mathbf{x})})}$ can be interpreted and used as a loss function to train a deep neural network. Here, $\mathbf{x}$ refers to the input variable, $f:\mathbf{x}\rightarrow y$ is the true labeling or mapping function, $\hat{f}$ is the estimated predictor, $f(\mathbf{x})=y$ is the true label or response variable, and $\hat{f}(\mathbf{x})=\hat{y}$ is the predicted output. $\sigma_{\mathbf{x},f(\mathbf{x})}$ and $\sigma_{\mathbf{x},\hat{f}(\mathbf{x})}$ denote the covariance matrices for the pairs of variables $\{\mathbf{x},f(\mathbf{x})\}$ and $\{\mathbf{x},\hat{f}(\mathbf{x})\}$, respectively. Fig.~\ref{fig:TRE_loss_geometry} depicts an illustrative explanation. 

Before presenting our methodology in both multi-source domain adaptation and continual learning, we show three appealing properties associated with $\sqrt{J_{vN}}$ (see supplementary material for proofs and empirical justifications):
\begin{itemize}
    \item $\sqrt{J_{vN}}$ has an analytical gradient and is automatically differentiable;
    \item Compared with the mean square error (MSE) loss, $\sqrt{J_{vN}(\sigma_{\mathbf{x},f(\mathbf{x})}:\sigma_{\mathbf{x},\hat{f}(\mathbf{x})})}$ enjoys improved robustness. 
    \item Compared with the cross-entropy (CE) loss, $\sqrt{J_{vN}(\sigma_{\mathbf{x},f(\mathbf{x})}:\sigma_{\mathbf{x},\hat{f}(\mathbf{x})})}$ satisfies the triangle inequality. That is, given three models $f_1$, $f_2$ and $f_3$, we have:\\
    $\sqrt{J_{vN}(\sigma_{\mathbf{x},f_1(\mathbf{x})}:\sigma_{\mathbf{x},f_2(\mathbf{x})})}\leq \sqrt{J_{vN}(\sigma_{\mathbf{x},f_1(\mathbf{x})}:\sigma_{\mathbf{x}, f_3(\mathbf{x})})}+ \sqrt{J_{vN}(\sigma_{\mathbf{x},f_3(\mathbf{x})}:\sigma_{\mathbf{x},f_2(\mathbf{x})})}$.
\end{itemize}

\begin{figure}
 \centering
\includegraphics[width=0.5\linewidth]{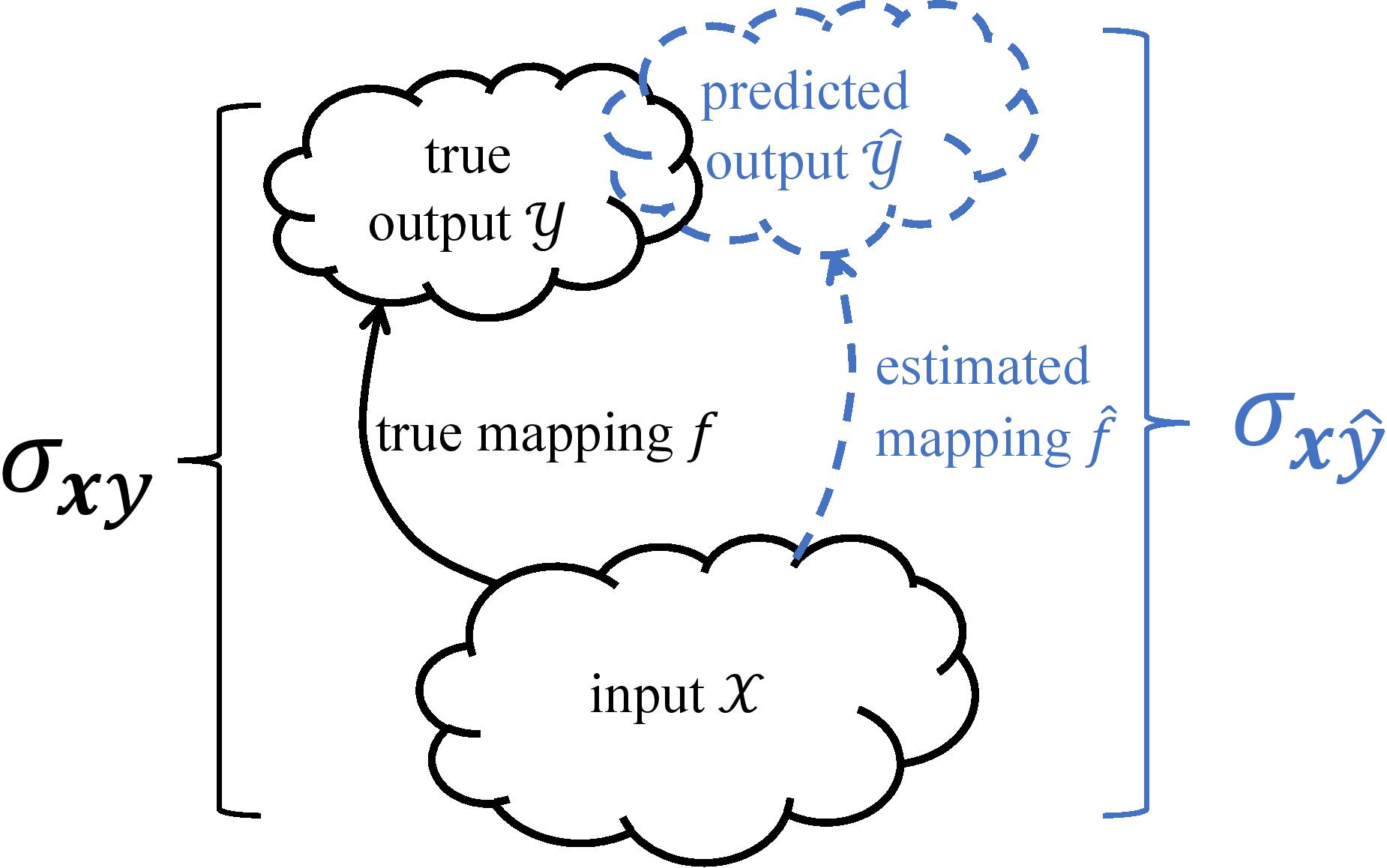}
\caption{The geometry of loss $\mathcal{L}:\mathbb{S}_{++}^p\times\mathbb{S}_{++}^p\rightarrow\mathbb{R}_+$: our $\sqrt{J_{vN}}$ searches for an ``optimal" predictor $\hat{f}$ that minimizes the discrepancy between the two covariance matrices $\sigma_{\mathbf{x},f(\mathbf{x})}$ and $\sigma_{\mathbf{x},\hat{f}(\mathbf{x})}$.}
\label{fig:TRE_loss_geometry}
\end{figure}

\section{MSDA by Matrix-based Discrepancy Distance}
\subsection{Bounding the von Neumann Conditional Divergence in Target Domain}

Motivated by the discrepancy distance $D_{disc}$ \cite{cortes2014domain} based on a loss function $\mathcal{L}:\mathcal{Y}\times\mathcal{Y}\rightarrow\mathbb{R}_+$, we first present our matrix-based discrepancy distance $D_{\text{M-disc}}$ to quantify the discrepancy between two distributions $P$ and $Q$ over $\mathcal{X}$ based on our new loss $\mathcal{L}:\mathbb{S}_{++}^p\times\mathbb{S}_{++}^p\rightarrow\mathbb{R}_+$ (i.e., $\sqrt{J_{vN}(\sigma_{\mathbf{x},f(\mathbf{x})}:\sigma_{\mathbf{x},\hat{f}(\mathbf{x})})}$).

\begin{definition}{}\label{def:matrix-based_hypothesis_divergence}
	The matrix-based discrepancy distance ($D_{\text{M-disc}}$) measures the longest distance between two domains (with respect to the hypothesis space $\mathcal{H}$) in a metric space equipped with the square root of Jeffery von Neumann divergence $J_{vN}$ as a distance function. Given domains $D_s$ and $D_t$ and their corresponding distributions $P_s$ and $P_t$, for any two hypotheses $h,h'\in \mathcal{H}$,  $D_{\text{M-disc}}$ takes the form:
	\begin{equation}
	\begin{aligned}
	D_{\text{M-disc}}(P_s,P_t) = \max_{h,h^\prime \in \mathcal{H}} \Big| \sqrt{J_{vN}(\sigma^s_{x,h(x)}:\sigma^s_{x,h^\prime(x)})} - \sqrt{J_{vN}(\sigma^t_{x,h(x)}:\sigma^t_{x,h^\prime(x)})} \Big|, 		
	\label{eq:quantumD_H}
	\end{aligned}
	\end{equation}
	with $a\in \{s,t\}$ and $g \in \{h,h^\prime\}$, the matrix $\sigma^a_{\mathbf{x},g(\mathbf{x})}$ is the covariance matrix for the pair of variable ${\mathbf{x},g(\mathbf{x})}$ in domain $D_a$. 
\end{definition}

Same to the notable $\mathcal{H}\Delta\mathcal{H}$ divergence in binary classification~\cite{ben2010theory}, $D_{\text{M-disc}}$ reaches the maximum value if a predictor $h'$ is very close to $h$ on the source domain but far on the target domain (or vice-versa). 
When fixing $h$, $D_{\text{M-disc}}(P_s,P_t;h)$ simply searches only for $h^\prime \in \mathcal{H}$ maximizing Eq.~(\ref{eq:quantumD_H}). The following theorem presents a new generalization upper bound for the square root of $J_{vN}$ on the target domain with respect to that of multiple sources.

\begin{theorem}
	\label{theorem:Q_{B}_bound}	
	Let $S=\{D_{s_1},\dots,D_{s_K}\}$ be the a set of $K$ source domains, and denote the ground truth mapping function in $D_{s_i}$ as $f_{s_i}$. Assign the weight $w_i$ to source $D_{s_i}$ (subject to $\sum_{i=1}^{K}{w_i=1}$) and generate a weighted source domain $D_\alpha$, such that the source distribution $P_\alpha=\sum_{i=1}^{K}{w_i P_{s_i}}$ and the mapping function $f_\alpha: x\rightarrow \left(\sum_{i=1}^{K}w_i P_{s_i}(x)f_{s_i}(x)\right)/\left(\sum_{i=1}^{K}w_i P_{s_i}(x)\right)$. For any hypothesis $h\in \mathcal{H}$, the square root of $J_{vN}$ on the target domain $D_t$ is bound in the following way:
{\fontsize{9.pt}{10.8pt} \selectfont 
	\begin{equation}
	\begin{aligned}
	\sqrt{J_{vN}(\sigma^t_{x,h(x)}:\sigma^t_{x,f_t(x)})} \leq \sum_{i=1}^K w_i \left( \sqrt{J_{vN}(\sigma^{s_i}_{x,h(x)}:\sigma^{s_i}_{x,f_{s_i}(x)})} \right) + D_{\text{M-disc}}(P_t,P_{\alpha};h) + \eta_{Q}(f_\alpha,f_t),
	\end{aligned}
	\end{equation}
}
where $\eta_{Q}(f_\alpha,f_t)= \min_{h^{*} \in \mathcal{H}} \sqrt{J_{vN}(\sigma^t_{x,h^{*}(x)}:\sigma^t_{x,f_t(x)})} +  \sqrt{J_{vN}(\sigma^\alpha_{x,h^{*}(x)}:\sigma^\alpha_{x,f_{\alpha}(x)})}$ is the minimum joint empirical losses on the combined source $D_\alpha$ and the target $D_t$, achieved by an optimal hypothesis $h^{*}$.
\end{theorem}

The result presented in Theorem~\ref{theorem:Q_{B}_bound} can be interpreted as bounding the square root of $J_{vN}$ on the target domain $D_t$ by quantities controlled by \textit{(i)} a convex combination over the square root of $J_{vN}$ in each of the sources, i.e., $\sqrt{J_{vN}(\sigma_{x,h(x)}^{s_i}:\sigma_{x,f_{s_i}(x)}^{s_i})}$; \textit{(ii)} the mismatch between the weighted distribution $P_\alpha$ and the target distribution $P_t$, i.e., $D_{\text{M-disc}}(P_t,P_\alpha;h)$; and \textit{(iii)} the optimal joint empirical risk on source and target, i.e., $\eta_Q(f_\alpha,f_t)$.
\textcolor{black}{The last term is irrelevant to the optimization and is expected to be small~\cite{zhao2019learning}. Notice that $\eta_Q$ is constant and only depends on $h^*$ in the case of a single source. For multiple source domains, the quantity $\eta_Q$ does include the weights $\textbf{w}$, yet it is constant for a given $\textbf{w}$.}

\subsection{Optimization by Adversarial Min-Max Game}

Similar to the notable Domain-Adversarial Neural Networks (DANN)~\cite{ganin2016domain} that implicitly performs distribution matching by an adversarial min-max game, we explicitly implement the idea exhibited by Theorem~\ref{theorem:Q_{B}_bound} and combine a feature extractor $f_\theta:\mathcal{X}\rightarrow\mathcal{T}$ and a class of predictor $\mathcal{H}:\mathcal{T}\rightarrow\mathcal{Y}$ in a unified learning framework:
\begin{align}
    \min\limits_{\substack{f_\theta,h\in \mathcal{H}\\ ||\mathbf{w}||_1=1}} \max\limits_{h^{\prime} \in \mathcal{H}} & \left(  \sum_{i=1}^{K} w_i \sqrt{J_{vN}(\sigma^{s_i}_{x,h(f_\theta(x))}:\sigma^{s_i}_{x,y})} + \right.  \nonumber\\ 
     & \left. \Big|\sqrt{J_{vN}(\sigma^t_{f_\theta(\mathbf{x}),h(f_\theta(\mathbf{x}))}:\sigma^t_{f_\theta(\mathbf{x}),h^\prime(f_\theta(\mathbf{x}))})}	 -\sum_{i=1}^{K}	
	w_k \sqrt{J_{vN}(\sigma^{s_i}_{f_\theta(\mathbf{x}),h(f_\theta(x))}:\sigma^{s_i}_{f_\theta(\mathbf{x}),h^\prime(f_\theta(x))})}\Big| \right) .
	\label{eq:MDD_objective}
\end{align}
The first term of Eq.~(\ref{eq:MDD_objective}) enforces $h$ to be a good predictor on all source tasks\footnote{In practice, one can replace the $J_{vN}$ loss with the root mean square error (RMSE) loss.}; the second term is an explicit instantiation of our $D_{\text{M-dist}}(P_t,P_\alpha)$. The general idea is to find a feature extractor $f_\theta(\mathbf{x})$ that for any given pair of hypotheses $h$ and $h'$, it is hard to discriminate the target domain $P_t$ from $P_\alpha$, the weighted combination of the source distributions.

We term our method the multi-source domain adaptation with matrix-based discrepancy distance (MDD) (pseudo-code in the supplementary material). We also noticed that a similar min-max training strategy has been used in~\cite{pei2018multi,saito2019semi,richard2020unsupervised}.

\subsection{Comparison with State-of-the-Art Methods}
\textcolor{black}{
We evaluate our MDD on four real-world datasets \textit{(i)} Amazon review dataset\footnote{\url{https://www.cs.jhu.edu/~mdredze/datasets/sentiment/}}, \textit{(ii)} TRANCOS which is a public benchmark for extremely overlapping vehicle counting, \textit{(iii)} the YearPredictionMSD data \cite{bertin2011yearpredictionmsd}, and 
\textit{(iv)} the relative location of CT slices on the axial axis dataset~\cite{graf20112d}. We keep the description and results of the last two datasets in the supplementary material.}

The following six methods are used for comparison:
\textit{(1)} DANN~\cite{ganin2016domain} is used by merging all sources into a single one; \textit{(2)} MDAN-Max and \textit{(3)} MDAN-Dyn, where MDAN refers to the multisource domain adversarial networks by~\cite{zhao2018adversarial}. It also applies a weighting scheme to all sources.
\textit{(4)} Adversarial Hypothesis-Discrepancy Multi-Source Domain Adaptation (AHD-MSDA)~\cite{richard2020unsupervised} and its baseline \textit{(5)} AHD-1S that merges all sources into one and then applies AHD-MSDA between the single combined source and the target domain. (6) Domain AggRegation Network (DARN) \cite{wen2020domain} after implementing the automatically differentiable maximum eigenvalue computation for the discrepancy computation.

In the first experiment, following \cite{richard2020unsupervised}, we employ a shallow neural network with two fully-connected hidden layers of size $500$ with ReLU activation, and a dropout rate of $10\%$. The Adam optimizer is used with learning rate $lr=0.001$, and batch size of $300$. We use $30$ training epochs, and perform $5$ independent runs. Each domain is used once as target and the remaining as sources.

The Amazon review dataset is introduced in~\cite{blitzer2007biographies}; it contains review texts and ratings of bought products. Products are grouped into categories. Following \cite{zhao2018adversarial,richard2020unsupervised}, we perform tf-idf transformation and select the top $1,000$ frequent words. Ratings are used as the target labels.

\begin{table*}
	\caption{Performance comparison in terms of mean absolute error (MAE) over five iterations on the Amazon rating data (with standard error in brackets). The best performance is marked in boldface. The categories are abbreviated as follows, \textbf{ba}:baby, \textbf{be}:beauty, \textbf{ca}:camera\&photo, \textbf{co}:computer\&video-games, \textbf{al}:electronics, \textbf{go}:gourmet-food, \textbf{gr}:grocery.}\label{tab:Amazon}
	\centering
	\begin{tabular}{p{0.005\textwidth}p{0.114\textwidth}p{0.08\textwidth}p{0.114\textwidth}p{0.114\textwidth}|p{0.114\textwidth}p{0.114\textwidth}|p{0.115\textwidth}}
		\toprule
        &AHD&DANN&AHD-&DARN&\multicolumn{2}{c|}{MDAN}&MDD\\
		&-1S&-1S&MSDA&&-Max&-Dyn& \\
		\midrule
		ba&0.627 \small{(.003)}&2.9 \small{(1.3)}&0.586 \small{(.003)}&0.755 \small{(.001)}&0.591 \small{(.015)}&0.711 \small{(.006)}&\textbf{0.581} \small{(.003)}\\
		be&0.614 \small{(.003)}&1.1 \small{(.2)}&0.608 \small{(.005)}&0.69 \small{(.001)}&0.628 \small{(.003)}&0.656 \small{(.004)}&\textbf{0.588} \small{(.003)}\\
		ca&0.559 \small{(.003)}&1.0 \small{(.1)}&0.534 \small{(.006)}&0.643 \small{(.002)}&0.522 \small{(.005)}&0.598 \small{(.006)}&\textbf{0.508} \small{(.003)}\\
		co&0.617 \small{(.005)}&2.2 \small{(.8)}&0.61 \small{(.004)}&0.665 \small{(.001)}&0.682 \small{(.016)}&0.829 \small{(.055)}&\textbf{0.584} \small{(.003)}\\
		el&0.669 \small{(.002)}&0.7 \small{(.01)}&0.657 \small{(.002)}&0.776 \small{(.000)}&0.654 \small{(.001)}&0.670 \small{(.003)}&\textbf{0.65} \small{(.001)}\\
		go&0.585 \small{(.002)}&0.9 \small{(.3)}&0.566 \small{(.003)}&0.639 \small{(.002)}&0.552 \small{(.003)}&0.553 \small{(.003)}&\textbf{0.537} \small{(.003)}\\
		gr&0.543 \small{(.003)}&1.5 \small{(.8)}&0.527 \small{(.002)}&0.627 \small{(.002)}&0.519 \small{(.002)}&0.538 \small{(.003)}&\textbf{0.513} \small{(.009)}\\
		\hline
	\end{tabular}
\end{table*}

\textcolor{black}{
The TRaffic ANd COngestionS (TRANCOS) \cite{guerrero2015extremely} dataset is a public benchmark dataset for extremely overlapping vehicle counting with $1,244$ images and $46,700$ manually annotated vehicles via the dotting method \cite{lempitsky2010learning}. It contains images that were collected from 11 video surveillance cameras. We apply hierarchical clustering to formulate five domains over the cameras. The hourglass network \cite{newell2016stacked} is used such that the encoder plays the role of the feature extractor, and the predictor and discriminator follow the decoder design. The predicted vehicle count is computed by integrating over the predicted density map after applying the ground truth mask, thereafter, the mean absolute error is computed on the predicted count. See the supplementary material for more details. 
The quantitative results on these two datasets are summarized in Table~\ref{tab:Amazon} and Table~\ref{tab:Camer_counting}, respectively. Our MDD always achieves the smallest mean absolute error on all target domains, except for "Dom2" of the counting problem. It is worth mentioning that DARN fails to generalize on source domains of TRANCOS and, hence, performs poorly on the target domain, as discussed in the supplementary material. 
}

We also analyse the weights $\mathbf{w}$ learned by our MDD (plots and discussion in supplementary material). In general, our learned weights  reflect the strength of relatedness from each source to the target. Moreover, we observe that our weights are much more stable across training epochs, whereas the weights learned by DARN always oscillate and are less linked in successive epochs.

\begin{table*}[t]
	\caption{Performance comparison in terms of mean absolute error (MAE) over three iterations on TRANCOS data (with standard error in brackets). The best performance is marked in boldface. DARN fails to generalize on the source domains, hence, performs very poorly on the target domains.}\label{tab:Camer_counting}
	  \centering
	\begin{tabular}{p{0.04\textwidth}cccc|cc|c}
		\toprule
		&AHD&DANN&AHD-&DARN&\multicolumn{2}{c|}{MDAN}&MDD\\
		&-1S&-1S&MSDA&&-Max&-Dyn& \\
		\midrule
		Dom1&46.87 \small{(12.89)}&16.19 \small{(0.42)}&57.19 \small{(22.93)}&---&32.17 \small{(7.98)}&29.35 \small{(3.96)}&\textbf{14.73} \small{(0.52)}\\
		Dom2&27.39 \small{(4.8)}&21.7 \small{(0.86)}&33.8 \small{(6.51)}&---&18.02 \small{(0.34)}&\textbf{14.34}\small{(0.24)}&15.27 \small{(0.92)}\\
		Dom3&63.69 \small{(31.62)}&28.43 \small{(5.63)}&63.27 \small{(24.77)}&---&38.5 \small{(11.77)}&26.81 \small{(4.61)}&\textbf{24.67} \small{(3.43)}\\
		Dom4&23.02 \small{(3.71)}&21.54 \small{(5.64)}&88.07\small{(52.72)} &---&19.89 \small{(3.83)}&22.86 \small{(1.04)}&\textbf{14.25} \small{(1.64)}\\
		Dom5&65.89 \small{(22.71)}&57.12 \small{(29.74)}&38.02 \small{(11.7)}&---&57.28 \small{(36.24)}&22.73\small{(4.72)}&\textbf{17.34} \small{(1.43)}\\
		\hline
	\end{tabular}
\end{table*}

\subsection{Visualizing Domain Importance in Synthetic Data}

We further evaluate the ability of MDD to discover the correct strength of relatedness from each source on a synthetic data, in which the ``ground truth" of relatedness is known. We construct a synthetic data set with six domains each with features from $\mathbf{x} \in [-1,1]^{12}$, and the Friedman target function \cite{friedman1991multivariate} $y(\mathbf{x}) = 10\sin(\pi x_1 x_3) + 20(x_5-0.5)^2 +10x_7 + 5x_9 + \epsilon$, $\epsilon \sim \mathcal{N}(0,1)$.
The six generated domains are equally distributed in the diagonal of the space $[-1,1]^{12}$. To this end, each domain $s_i \in \{s_1,\dots,s_6\}$ is sampled from $\mathcal{N}(\mu^{(i)},\Sigma^{(i)})$, such that $\mu^{(i)} = c_i \textbf{1}_{12}$, where $c_i =(-1+(2i-2)/5)$ and $\textbf{1}_{12}$ is the the all-one vector of size 12. The element of the covariance matrix $\Sigma^{(i)}$ are set to zero except for $\Sigma^{(i)}_{2i-1,2i}=\Sigma^{(i)}_{2i,2i-1}=0.1$, $\Sigma^{(i)}_{2i,2i+1}=\Sigma^{(i)}_{2i+1,2i}=0.07$ if $i<6$, $\Sigma^{(i)}_{2i-1,2i-2}=\Sigma^{(i)}_{2i-2,2i-1}=0.07$ if $i>1$, and 
$\Sigma^{(i)}_{2k-1,2k}=\Sigma^{(i)}_{2k,2k-1}=0.5$ where $k=i-1$ or $i+1$. This way, the neighboring domains will have a gradual covariate shift in terms of both mean and covariance. 

The distribution of the first two dimensions of $\mathbf{x}$ is depicted in Fig.~\ref{fig:MDD_synthetic_Data} and the covariance matrix $\Sigma^{(i)}$ for domain $i$ is illustrated in Fig.~\ref{fig:MDD_synthetic_Sigma}. 
Fig.~\ref{fig:MDD_synthetic_DARN} to \ref{fig:MDD_synthetic_MDD} show the weights learned by DARN, AHD-MSDA and MDD, respectively. The value in the $(i,j)$-th entry is the weight from source $j$, when the target is domain $i$. As can be seen, our MMD learns an almost symmetric weight matrix with high weights centered around the diagonal and smoothly fading weights in the anti-diagonal direction. AHD-MSDA seems to learn uniform weights. DARN learns sparse weights while often failing in ranking the sources in agreement with the ground truth. 


\begin{figure*}
     \centering
     \begin{subfigure}[b]{0.17\textwidth}
         \centering
         \includegraphics[width=1.1\textwidth]{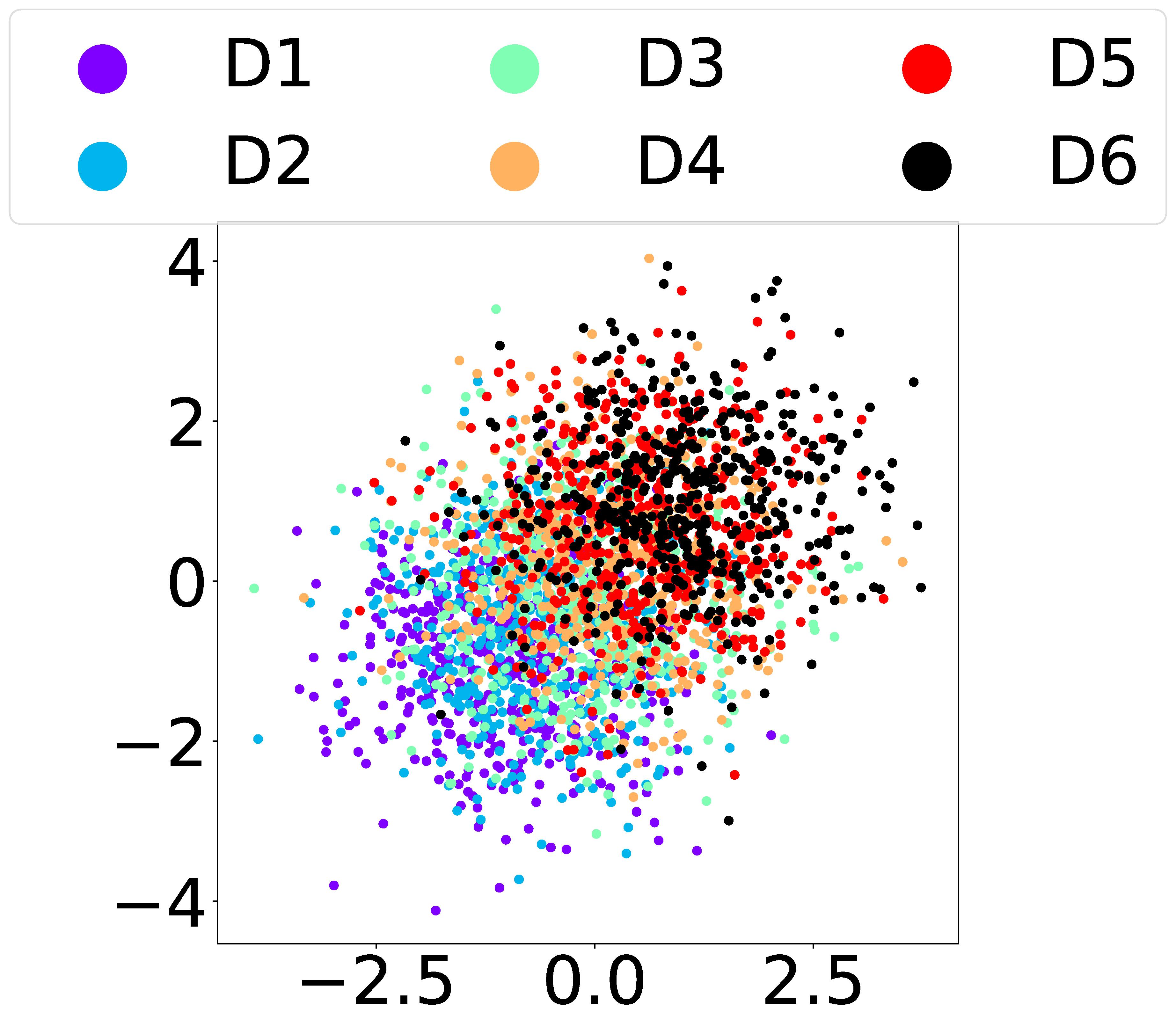}
         \caption{Data}
         \label{fig:MDD_synthetic_Data}
     \end{subfigure}
     \hfill
     \begin{subfigure}[b]{0.2\textwidth}
         \centering
         \includegraphics[width=0.9\linewidth]{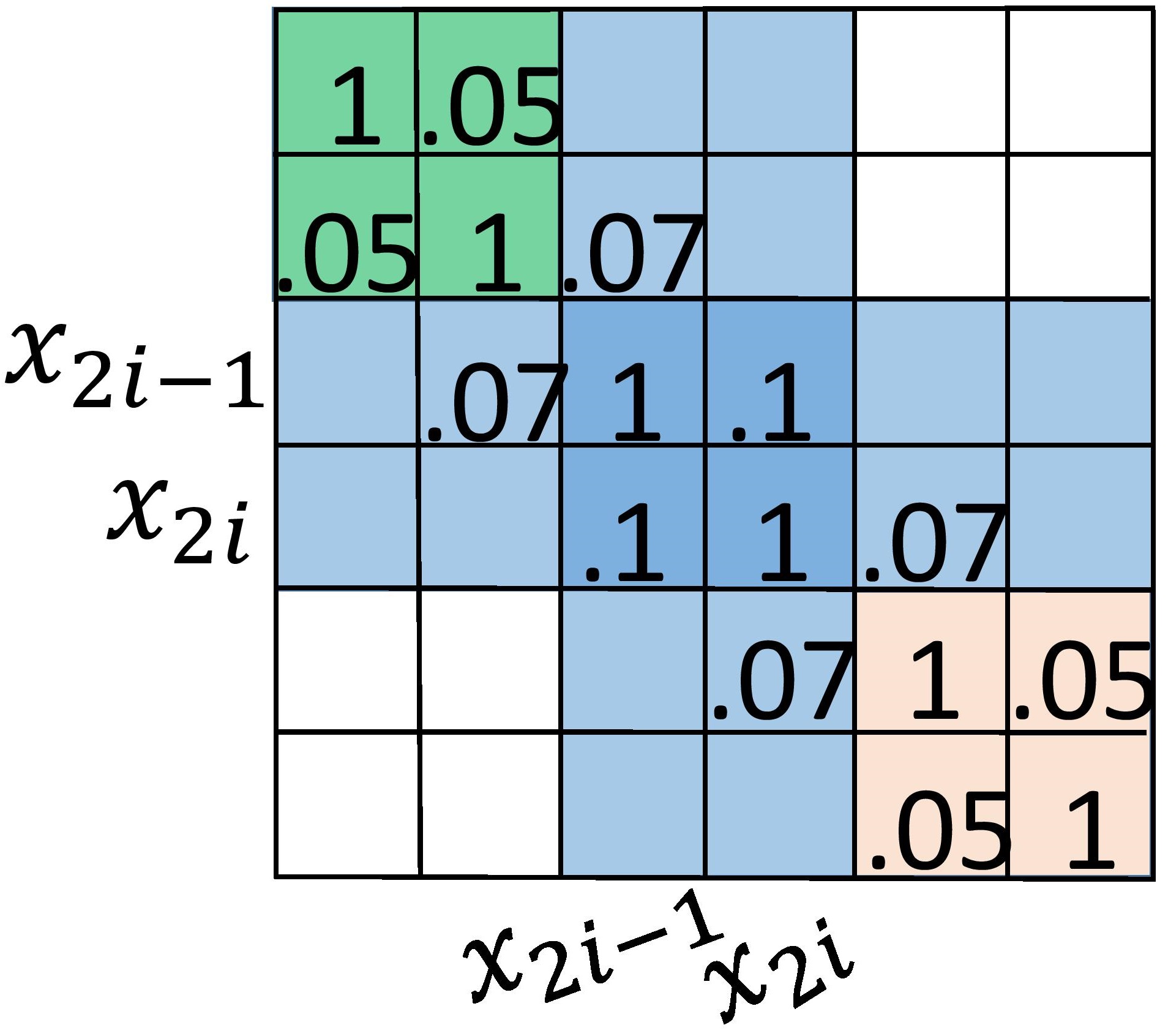}
         \caption{Submatrix of $\Sigma^{(i)}$}
         \label{fig:MDD_synthetic_Sigma}
     \end{subfigure}
     \hfill
     \begin{subfigure}[b]{0.195\textwidth}
         \centering
         \includegraphics[width=1.05\textwidth]{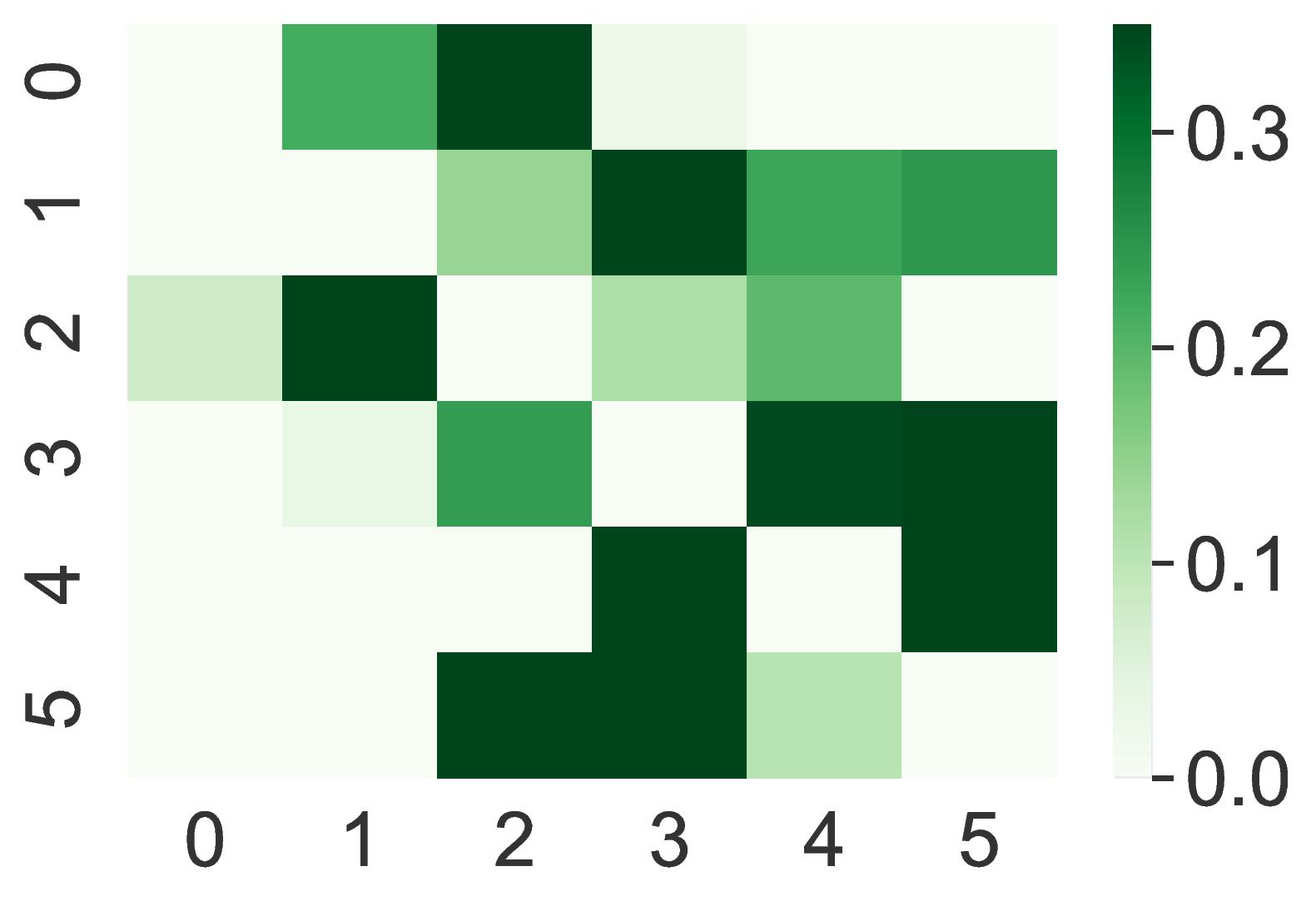}
         \caption{DARN}
         \label{fig:MDD_synthetic_DARN}
     \end{subfigure}
     \hfill
     \begin{subfigure}[b]{0.195\textwidth}
         \centering
         \includegraphics[width=1.05\textwidth]{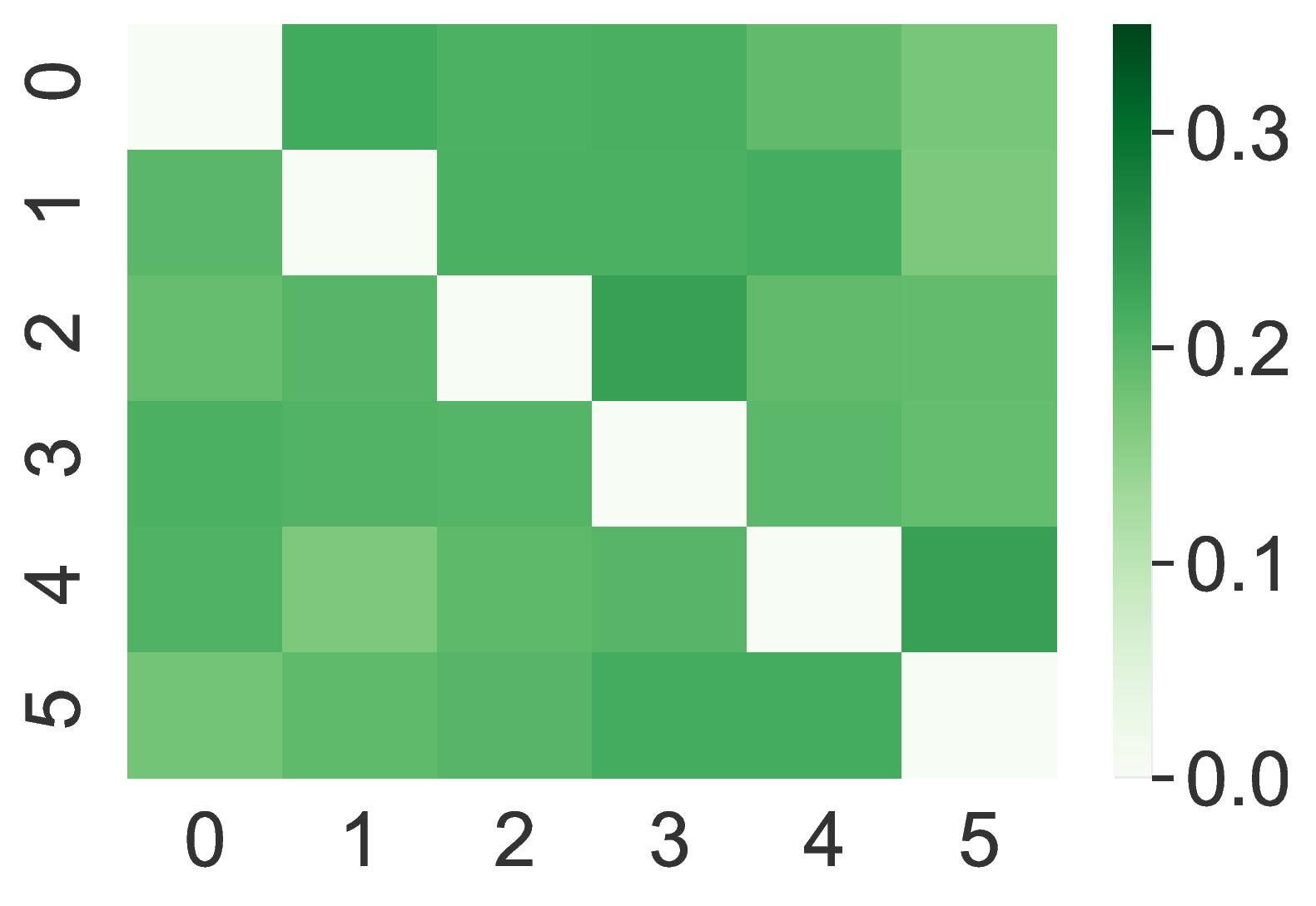}
         \caption{AHD-MSDA}
         \label{fig:MDD_synthetic_AHD-MSDA}
     \end{subfigure}
     \hfill
     \begin{subfigure}[b]{0.195\textwidth}
         \centering
         \includegraphics[width=1.05\textwidth]{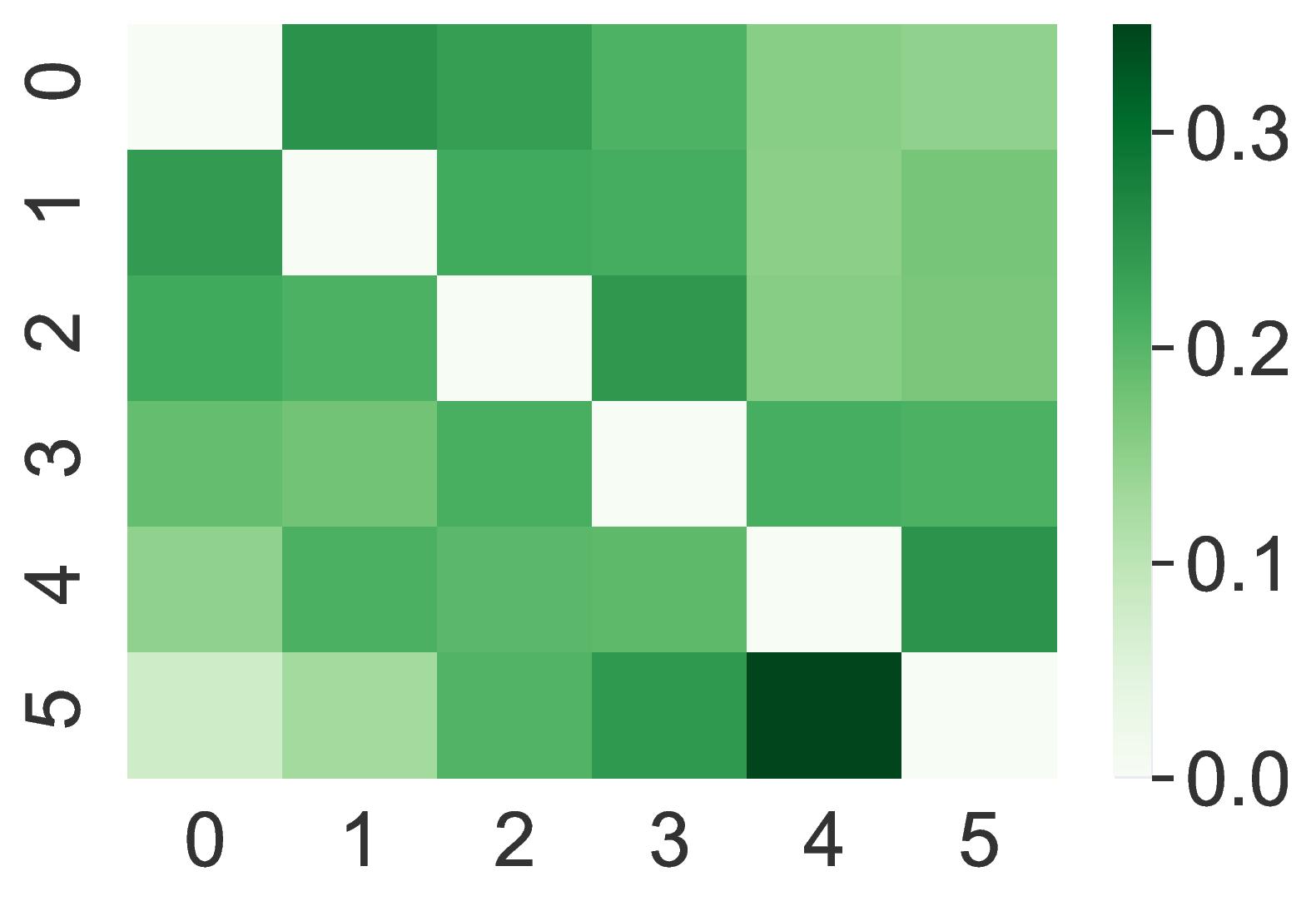}
         \caption{MDD}
         \label{fig:MDD_synthetic_MDD}
     \end{subfigure}
        \caption{(a) The data used for the weight learning over synthetic data. (b) The submatrix of the covariance matrix $\Sigma^{(i)}$ of domain $i$. Figures (c), (d), and (e) show heatmaps of the learned weights of DARN, AHD-MSDA, and MDD, respectively. The rows and columns represent the target and source domains, respectively.}
        \label{fig:MDD_synthetic}
\end{figure*}

\section{Continual Learning by Representation Similarity Penalty}

We demonstrate, in this section, that the von Neumann conditional divergence is also suitable to alleviate negative backward transfer or catastrophic forgetting in continual learning (CL). We exemplify our argument by proposing a new regularization-based CL approach.

\subsection{Elastic Weight Consolidation (EWC) and its Extensions}\label{sec:EWC}

Regularization approaches mitigate catastrophic forgetting by imposing penalties on the updates of the important neural weights (to previous tasks)~\cite{parisi2019continual,de2019continual}. 
As a notable example in this category, EWC~\cite{kirkpatrick2017overcoming} consists of a quadratic penalty on the difference between the parameters $\theta$ for the old and the new tasks. The objective to be minimized when observing task $T_B$ after learning on task $T_A$ is:
\begin{equation}
\mathcal{L}(\theta) = \mathcal{L}_B(\theta) + \sum_i \frac{\lambda}{2} \mathcal{F}_{\theta_i} (\theta_i - \theta_{A,i}^*)^2,
\label{eq:EWC}
\end{equation}
$\mathcal{L}_B(\theta)$ is the loss for task $T_B$, $\lambda$ is the regularization strength, $\{\theta_{A,i}^*\}$ is the set of parameters after learning on task $A$, and $F_\theta$ is the diagonal Fisher information matrix (FIM).
The $i$-th diagonal element of $F_\theta$ is computed as $F_{\theta_i}=\mathbb{E}[(\frac{\partial\mathcal{L}}{\partial \theta_i})^2]$.
The supplementary material shows the derivation of Eq.~(\ref{eq:EWC}).

EWC assumes all weights in $\theta$ are independent, which leads to a diagonal FIM. To make this assumption more practical, R-EWC~\cite{liu2018rotate} takes a factorized rotation of parameter space that leads to the desired diagonal FIM. \cite{chaudhry2018riemannian} reformulates the objective of EWC by KL-divergence in the Riemannian Manifold and suggests an efficient and online version of EWC. As an alternative to computing FIM, synaptic intelligence (SI)~\cite{zenke2017continual} measures each parameter's importance by its accumulative contribution to the loss changes.

\subsection{Measuring Weight Significance by Representation Similarity}

In this section, we introduce a new form of regularization that measures the significance of a group of weights (rather than individual ones) to $T_A$ by the (dis)similarity of local representations between $T_A$ and $T_{B}$ induced by these weights. Our method's essence comes from observing that tasks with similar representations are more prone to overwrite or negatively affect each other. A similar observation has been recently discovered by~\cite{ramasesh2020anatomy}.

Specifically, in the $d$-th hidden layer, suppose we identified $K$ groups of neurons ($g_1^d, g_2^d, \cdots, g_K^d$) that are functionally mutually independent. Each group can be viewed as a module that operates independently. Therefore, changes to parameters belonging to the same module should be regularized together taking into account \textit{(i)} their relatedness to the different tasks (through the von Neumann conditional divergence), and \textit{(ii)} the parameter’s interdependence through the network modularization. Taking these two aspects into consideration, we define a new regularization-based CL objective as:
\begin{align}
	\mathcal{L}(\theta) &= \mathcal{L}_{T_B}(\theta) + \sum_{T_A \in \mathbb{T}\setminus\{T_B\}} \sum_{k,d} r_{k,d}^{A} \sum_{\theta_i \in g_k^d}(\theta_i - \theta_{T_A,i}^*)^2	, 		\label{eq:ModularEWC1} \\
	r_{k,d}^A &=	\frac{1}{Z} \frac{\lambda}{2} \exp(-D(P_{T_A}(y|g_k^d(x)):P_{T_B}(y|g_k^d(x)))). 		
	\label{eq:ModularEWC2}
\end{align}
Objective~(\ref{eq:ModularEWC1}) iterates over each group $g_k^d$ (second sum), and computes the representation similarity~(\ref{eq:ModularEWC2}), induced by the sub-network associated by the group of neurons $g_k^d$, between the current task $T_B$ and each previous task $T_A \in \mathbb{T}\setminus\{T_B\}$. This similarity takes the form of the softmax of the negative divergence with $Z$ being the normalization term, and $D$ is the symmetric von Neumann conditional divergence, i.e., Eq.~(\ref{eq:VN_conditional_divergence_sym}). Based on this similarity, the change in the parameters of each group $g_k^d$ is penalized by the representation indifference between the two tasks caused by that group. Hence, we call our method representation similarity penalty (RSP).
For an architecture with $R$ layers, RSP computes the groups for layers $d \in \{2,\dots,R-1\}$, which leaves the parameters and bias of the first layer without assigned groups; for these parameters the Fisher index is used to weight the penalty.

\subsection{Implementation Details and Empirical Evaluation}


RSP employs the modularization strategy in~\cite{watanabe2018modular} to construct groups of neurons in each layer that are mutually independent. In our experiments, we fix the number of groups to be $K_d=20$.


\subsubsection{Setting, Datasets and Performance Measures}
The following empirical evaluations follow the continual learning setting described in \cite{riemer_learning_2018}, where each sample of each task is observed in a single pass sequence. As for the neural network architecture, we use a single head fully-connected neural network with two hidden layers, each with $100$ neurons, a $28 \times 28$ input layer, and an output layer with a single head with $10$ units. This architecture is similar to the one used in \cite{lopez-paz_gradient_2017}. The hidden layers employ the ReLU activation, and SGD is used to minimize the softmax cross-entropy on the online training data.

We evaluate on the following datasets: \textit{(i)}  MNIST Permutations \textbf{(mnistP)} \cite{kirkpatrick2017overcoming}, \textit{(ii)} MNIST Rotations \textbf{(mnistR)} \cite{lopez-paz_gradient_2017},	\textit{(iii)} Permuted Fashion-MNIST \textbf{(fashionP)} \cite{xiao2017_online}, and
\textit{(iv)} Permuted notMNIST \textbf{(notmnistP)} \cite{bulatov_machine_2011}.
All these datasets contain images of size $28\times 28$ pixels. Additionally, we also perform a comparison on the Omniglot dataset~\cite{lake2011one} using the first ten alphabets and a convolutional neural network; the setting and results are explained in the supplementary material.

To measure the learnability and resistance to forgetting, we compute three performance measures: \textit{(i)} Learning accuracy (LA) is the average accuracy on each task after learning it. \textit{(ii)} Retained accuracy (RA) is the average performance on all tasks after observing the last one. \textit{(iii)} Backward transfer (BT)  represents the loss in performance due to forgetting, i.e., the difference between LA and RA~\cite{chaudhry2018riemannian}.

\begin{table*}[h]
    \centering
    \caption{Performance comparison between RSP, AGEM, MER, R-EWC and EWC. The numbers in parentheses are the standard errors (SE) of the means in the former row.
    D1: not-mnistP, D2: fashionP, D3: mnistR, D4: mnistP. BT is rounded to the nearest integer when it is larger than 10.}      
    \begin{tabularx}{0.9\textwidth}{p{0.01\textwidth} 
    p{0.03\textwidth} p{0.03\textwidth} p{0.045\textwidth}
    |p{0.03\textwidth} p{0.03\textwidth} p{0.044\textwidth}
    |p{0.03\textwidth} p{0.03\textwidth} p{0.045\textwidth}
    |p{0.03\textwidth} p{0.03\textwidth} p{0.045\textwidth} 
    |p{0.03\textwidth} p{0.03\textwidth} p{0.045\textwidth} }
        \hline
        & \multicolumn{3}{c}{AGEM} & \multicolumn{3}{c}{MER} & \multicolumn{3}{c}{R-EWC} & \multicolumn{3}{c}{EWC} & \multicolumn{3}{c}{RSP} \\
        \text{} & \text{RA} &  \text{LA} &  \text{BT}& \text{RA} &  \text{LA} &  \text{BT} & \text{RA} &  \text{LA} &  \text{BT} & \text{RA} &  \text{LA} &  \text{BT} & \text{RA} &  \text{LA} &  \text{BT} \\
        \cmidrule(lr){2-4}\cmidrule(lr){5-7}\cmidrule(lr){8-10}\cmidrule(lr){11-13}\cmidrule(lr){14-16}
        \multirow{2}{*}{D1}&66.6 & 78.5 & -12&50.6 & 55.1 & -4.6         &69.8&\textbf{83.8}&-14&68.7&81&-12&\textbf{72.3}&79.5&-7.2\\
        &\small{(1.5)} &\small{(0.6)} &\small{(1.5)}&\small{(0.7)} &\small{(0.8)} &\small{(0.7)}&\small{(0.5)}&\small{(0.1)}&\small{(0.5)}&\small{(0.3)}&\small{(0.1)}&\small{(0.2)}&\small{(0.3)}&\small{(0.1)}&\small{(0.2)}\\            
        \hline
        \multirow{2}{*}{D2}&59.5 & 65.4 & -5.9&53.3 & 61.2 & -7.8         &58.5&64.0&-5.4&42.2&56.2&-14&\textbf{62.5}&\textbf{66.6}&-4.2\\
        &\small{(0.5)} &\small{(0.3)} &\small{(0.5)}&\small{(0.1)} &\small{(0.8)} &\small{(0.9)}&\small{(0.8)}&\small{(0.1)}&\small{(0.7)}&\small{(2.1)}&\small{(1.4)}&\small{(0.8)}&\small{(0.3)}&\small{(0.1)}&\small{(0.3)}\\            
        \hline
        \multirow{2}{*}{D3}&75.0 & 85.6 & -11&\textbf{81.2} & 81.3 & -0.2         &60.9&\textbf{87.8}&-27&62.1&85.6&-24&62.9&83.6&-21\\
        &\small{(0.3)} &\small{(0.1)} &\small{(0.3)}&\small{(0.2)} &\small{(0.2)} &\small{(0.2)}&\small{(0.8)}&\small{(0.1)}&\small{(0.8)}&\small{(0.3)}&\small{(0.1)}&\small{(0.3)}&\small{(0.2)}&\small{(0.1)}&\small{(0.2)}\\          
        \hline
        \multirow{2}{*}{D4}&67 & 78.7 & -12&68.9 & 75.9 & -7.0         &64.8&79.1&-14&66.1&77&-12&\textbf{71.8}&\textbf{80.8}&-9\\
        &\small{(0.4)} &\small{(0.3)} &\small{(0.6)}&\small{(0.3)} &\small{(0.2)} &\small{(0.3)}&\small{(0.5)}&\small{(0.2)}&\small{(0.4)}&\small{(1.9)}&\small{(0.7)}&\small{(1.3)}&\small{(0.2)}&\small{(0.1)}&\small{(0.2)}\\        \hline
    \end{tabularx}
    \label{tb:exp1_comparisonEWC}
\end{table*}

\subsubsection{Comparison Protocol and Results}
We compare the performance of our RSP against that of EWC, R-EWC, and two popular replay-based CL methods, namely the Averaged Gradient Episodic Memory (AGEM)~\cite{Chaudhry:2019:ELL-A-GEM}, and the Meta-Experience Replay (MER) \cite{riemer_learning_2018}.
A grid-based hyperparameter search is carried on for each method on each dataset as explained in the supplementary material.
The ten datasets form a stream of ten tasks, each of which contains a sequence of only 1000 samples. Every time an evaluation is performed on a task, it is done on its test data of $10,000$ samples.

We employ the aforementioned online setting with a restricted memory budget of ten samples per task. Table~\ref{tb:exp1_comparisonEWC} shows that RSP outperforms all other methods in terms of RA on all data sets, except for mnistP. RSP also shows the highest LA on fashionP and mnistP. Only on mnistR, RSP performs worse than MER on RA, and worse than R-EWC on LA.

Compared only to EWC, RSP improves RA by $20\%$ on the fashionP, and around $6\%$ and $4\%$ on notmnistP and mnistP, respectively. In terms of LA, both methods perform similarly on notmnistP and mnistR, whereas RSP shows substantial improvement on fashionP and mnistP. This result indicates that RSP performs better than EWC in encouraging positive forward transfer under the circumstances of limited memory. The gain in both LA and RA that our modification causes to EWC is accompanied by less negative backward transfer (BT) on all datasets. Under the setting adopted in this experiment, R-EWC performs similarly or slightly better than EWC, but it is still worse than RSP in most cases.

\section{Related Work}

\textbf{Multi-Source Domain Adaptation (MSDA)}
Existing domain adaptation methods mainly focus on the single-source scenario. \cite{Mansour:2009:DALBA} assumes that the target distribution can be approximated by a mixture of given source distributions, which also partially motivated our MDD. There are other theoretical analyses to the design of MSDA methods, with the purpose of either developing more accurate measures of domain discrepancy or deriving tighter generalization bounds~\cite{redko2019advances,zhao2020multi}. Most existing bounds are based on the seminal work~\cite{blitzer2008learning,ben2010theory}. For example, \cite{zhao2018adversarial} extends the generalization bound in \cite{blitzer2008learning} to multiple sources. \cite{li2018extracting} considered the relationship between pairwise sources and derived a tighter bound on weighted multi-source discrepancy based on a Wasserstein-like metric. Calculating such pairwise weights can be computationally demanding when the number of sources is large.  Recently, \cite{wen2020domain} extends the upper-bound on the target domain loss, developed by~\cite{cortes2019adaptation}, to MSDA. The new bound depends on the discrepancy distance between two domains~\cite{Mansour:2009:DALBA}. 
\cite{richard2020unsupervised} uses the hypothesis distance for regression~\cite{cortes2014domain} and derives a similar bound. 





Distinct from these methods, our discrepancy measure does not align the distribution of feature $p(\mathbf{t})$. Rather, it aims to match the dependence between $\mathbf{t}$ and $y$ across domains, such that the conditional distributions $p(y|\mathbf{t})$ remain similar. To the best of our knowledge, we are also the first to derive a new generalization bound based on the matrix-based divergence~\cite{kulis2009low,yu2020measuring}.


\textbf{Regularization-based Continual Learning and Network Modularizaton} The general idea and popular regularization-based continual learning methods have been discussed in the previous section.
Recently, network modularization is becoming a popular paradigm for efficient network training~\cite{hadsell2020embracing,duan2021modularizing}. 
Indeed, biological brains are modular, with distinct yet interacting subsystems. Introducing modularization to prevent forgetting dates back to~\cite{pape2011modular} on the training of deep belief networks (DBN)~\cite{hinton2006fast}. Recently, \cite{Veniat:2021:ECLMN} suggests a modular solution by identifying the trained modules (groups of neurons) to be re-used and extending the network with new modules for each new task. 





\section{Conclusion}

{\color{black}
We introduced von Neumann conditional divergence $D_{vN}$ to align the dependence between latent representation $\mathbf{t}$ and response variable $y$ across different domains and exemplified this idea in domain adaptation, assuming multiple source tasks are observed either simultaneously or sequentially. For the former, we consider multi-source domain adaptation (MSDA) and developed a new generalization bound as well as a new learning objective based on the loss induced by $D_{vN}$. 
For the latter, we focus on continual learning (CL) and demonstrated that such dependence can be formulated as a penalty to regularize the changes of network parameters. Empirical results justify the superiority of our methods. 



Our point of departure is how learning, in general, can benefit from the conditional von Neumann divergence. 
At the same time, more than promoting a specific method, we aim at investigating a suitable distance measure for aligning representations. The perfect testbed for this is MSDA and CL. While the techniques we propose are deeply rooted and shaped by these domains, we hope them to be seen as an example of how the divergence can be beneficial.}








\clearpage
\newpage

\bibliographystyle{abbrv}
\bibliography{ijcai21.bib}

\clearpage
\newpage

This document contains the supplementary material for the \textit{``Learning to Transfer with von Neumann Conditional Divergence"} manuscript. It is organized into the following topics and sections:

\begin{enumerate}
\item Ethics Statement and Potential Societal Impacts
\item Proofs and Additional Remarks to the Jeffery von Neumann Divergence $J_{vN}$
\begin{enumerate}
\item $J_{vN}$ as a Loss Function
\item Differentiability of $\sqrt{J_{vN}(X;Y)}$
\item Triangle Inequality of $\sqrt{J_{vN}(\sigma_{\mathbf{x},f(\mathbf{x})}:\sigma_{\mathbf{x},\hat{f}(\mathbf{x})})}$
\item Robustness of $\sqrt{J_{vN}(\sigma_{\mathbf{x},f(\mathbf{x})}:\sigma_{\mathbf{x},\hat{f}(\mathbf{x})})}$
\end{enumerate}
\item Convergence Behavior of the Matrix-based von Neumann Divergence
\item Multi-Source Domain Adaptation with Matrix-based Discrepancy Distance
\item Further Note on EWC
\begin{enumerate}
\item Elastic Weight Consolidation
\item Special Relation to EWC and Fisher Information
\end{enumerate}
\item Illustrations and Complexity Analysis
\begin{enumerate}
    \item Multi-Source Domain Adaptation with Matrix-Based Discrepancy Distance
    \item Continual Learning by Representation Similarity Penalty
\end{enumerate}
\item Evaluation Details and Additional Experiments
\begin{enumerate}
\item Additional Results and Information for Evaluating the Multi-Source Domain Adaptation with Matrix-based Discrepancy Distance
\item Additional Results and Information for Evaluating the Continual Learning by Representation Similarity Penalty
\end{enumerate}
\end{enumerate}

\clearpage
\section{Ethics Statement and Potential Societal Impacts} 
The aim of our work is to enable domain adaptation, which in its essence aims at learning from fewer data. This has a significant positive effect on the environment by reducing computational power and runtime to train models, i.e., less electricity consumption and less $CO_2$ emissions. This applies to both types of approaches proposed in this work: (i) multi-source domain adaptation, and (ii) continual learning.
Moreover, we also show that our motivated loss function, $J_{vN}$, enjoys the robustness property, one of the main functional properties required to achieve fairness. Even though we don't establish this connection, we believe that this line of work would make a foundation to achieve fairness when AI methods are consulted.

\section{Proofs and Additional Remarks to the Jeffery von Neumann Divergence $J_{vN}$}
\label{Proofs_and_Additional}

\subsection{$\sqrt{J_{vN}(\sigma_{\mathbf{x},f(\mathbf{x})}:\sigma_{\mathbf{x},\hat{f}(\mathbf{x})})}$ as a Loss Function}
\label{Interpretability}
For simplicity, we consider the argument of the square root, i.e., $J_{vN}(\sigma_{x,y}:\sigma_{x,\hat{y}})$. Let us consider a regression scenario, in which $\mathbf{x}\in\mathbb{R}^d$ and $y\in\mathbb{R}$, then both the joint covariance matrices $\sigma_{x,y}$ and $\sigma_{x,\hat{y}}$ are symmetric positive definite and of size $(d+1)\times(d+1)$. At first, one should note that $\sigma_{x,y}$ differs from $\sigma_{x,\hat{y}}$ only in the first row and the first column associated with $y$ (or $\hat{y}$). This is just because the remaining elements of both matrices is the covariance matrix $\sigma_x$ that only depends on the input. See Figure~\ref{fig:joint_cov} for an illustration.

\begin{figure}[h]
     \centering
     \subfloat[$\sigma_{\mathbf{x},y}$]{         
        \includegraphics[width=.35\linewidth]{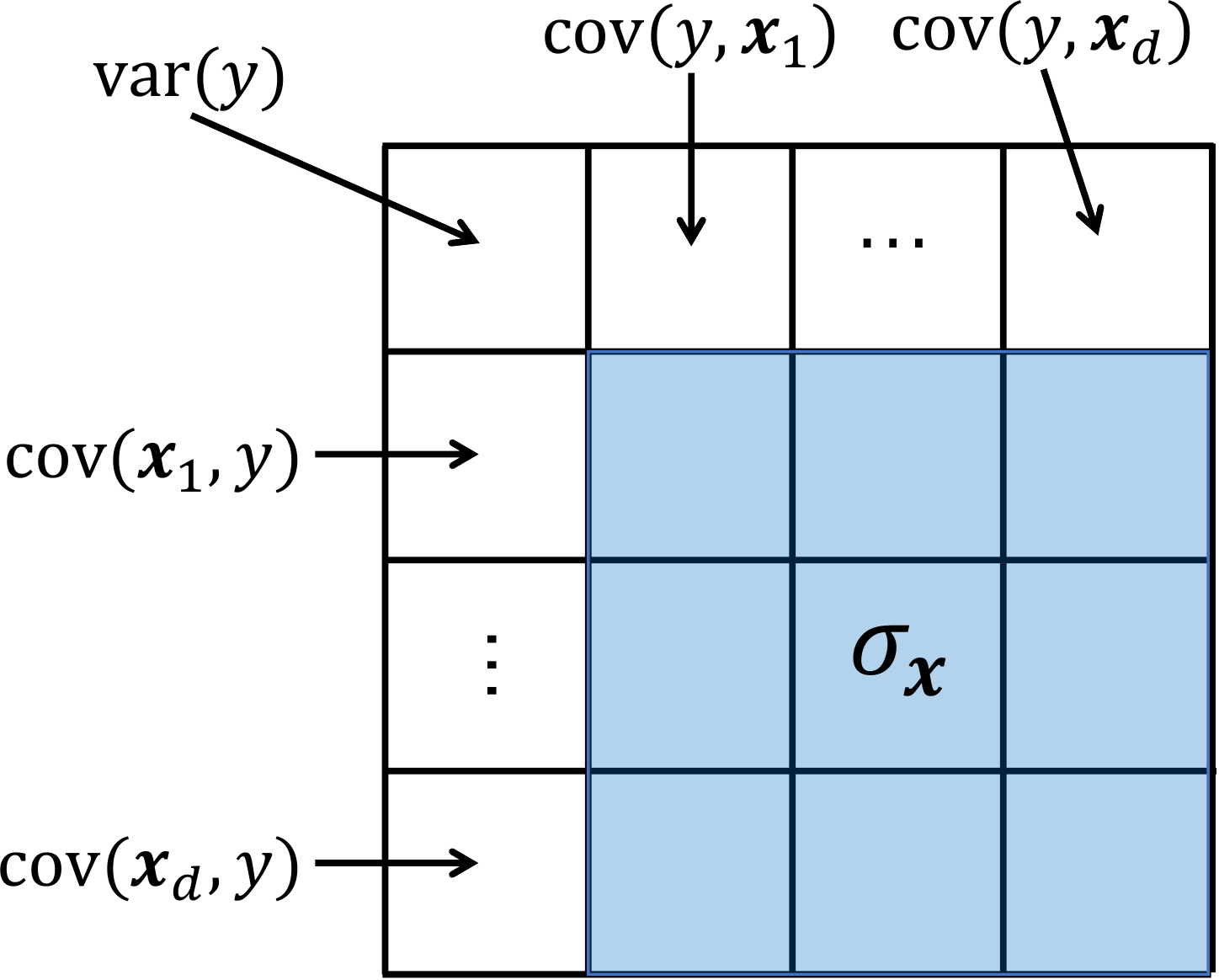}
     }
     \hspace{30pt}
     \subfloat[$\sigma_{\mathbf{x},\hat{y}}$]{         
        {\includegraphics[width=.35\linewidth]{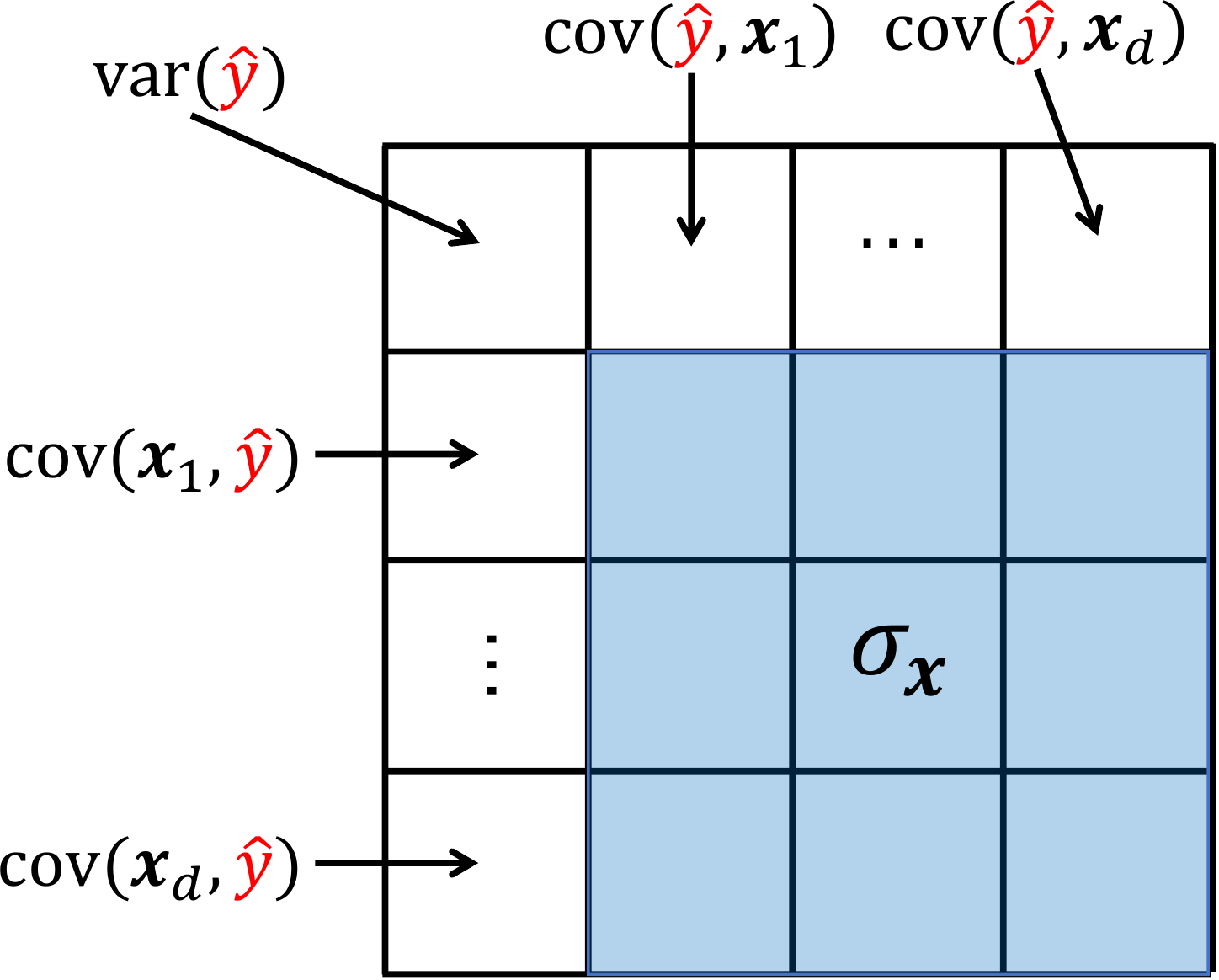}}
     }     
    \caption{The joint covariance matrices $\sigma_{\mathbf{x},y}$ (left) and $\sigma_{\mathbf{x},\hat{y}}$ (right). Two matrices only differ in the first column and row associated with $y$ or $\hat{y}$.}
    \label{fig:joint_cov} 
\end{figure}     

If we look deeper, the first row and column in $\sigma_{x,y}$ (or $\sigma_{x,\hat{y}})$ quantify the variance of $y$ (or $\hat{y}$) and the covariance between $y$ (or $\hat{y}$) and each dimension of $x$ (denote $x_i$ the $i$-th dimension of $x$). In this sense, our matrix-based loss reduces to zero if and only if \textit{(i)} the variance of $y$ and $\hat{y}$ are the same; and \textit{(ii)} for an arbitrary dimension $x_i$, the covariance $cov(y,x_i)$ is the same to the covariance $cov(\hat{y},x_i)$. On the other hand, suppose $y$ and $\hat{y}$ are Gaussian distributed\footnote{Note that, we did not make any distribution assumption on $p(y)$ or $p(\hat{y})$ when optimizing our objective. Here, we just take the Gaussian assumption for simplicity to build the connection between our loss and the classic cross-entropy loss.} with $y\sim N(\mu_y,\sigma_y)$ and $\hat{y}\sim N(\mu_{\hat{y}},\sigma_{\hat{y}})$, then the Kullback–Leibler (KL) divergence reduces to~\cite{cover1999elements}:
\begin{align}
    D_{KL}(p(y),p(\hat{y})) & =-\int{p(y)\log{\left(\frac{p(\hat{y})}{p(y)}\right)}dy}\nonumber \\
    & =\log{\frac{\sigma_{\hat{y}}}{\sigma_y}}+\frac{\sigma_y^2+{(\mu_y-\mu_{\hat{y}})}^2}{2\sigma_{\hat{y}}^2}.
\end{align}

If $y$ and $\hat{y}$ are mean centered, then the KL divergence only relies on the variance of $y$ and $\hat{y}$. Moreover, we have: 
\begin{align}\label{eq:KL_and_cross_entropy}
    D_{KL}(p(y),p(\hat{y}))&=-\int{p(y)\log{\left(\frac{p(\hat{y})}{p(y)}\right)}dy}\nonumber \\
    &=-\int{p(y)\log{\left(p(\hat{y})\right)}dy}+\int{p(y)\log{\left(p(y)\right)}dy}\nonumber \\
    &=H(p(y),p(\hat{y}))-H(p(y)).
\end{align}
The first term on the r.h.s. of Eq.~(\ref{eq:KL_and_cross_entropy}) is exactly the cross entropy, and the second term is the entropy of $p(y)$, a constant that only depends on the training data. In this context, we can view the cross-entropy and the KL divergence are optimizing the same quantity when they are used as loss functions.

To summarize, we can conclude that, in contrast to the popular KL divergence loss or cross-entropy loss that matches $p(y)$ to $p(\hat{y})$, our matrix-based loss adds an additional penalty on ${\cov(\hat{y},x_i)|}_{i=1}^d$. We know that the covariance can be interpreted as a linear dependence (although it is not upper bounded). In this sense, our matrix-based loss also encourages the dependence between each dimension of input and the predicted variable $\hat{y}$ matches to the ground truth. 
However, our loss has a limitation: it is less sensitive to the mean shift of $y$ or $\hat{y}$. That is, suppose the group-truth values are $y=\left[0~0~0~1~1~1\right]$, if our predictions are $\hat{y} = \left[1~1~1~2~2~2\right]$, our loss becomes zero whereas the prediction between $y$ and $\hat{y}$ has a bias term. This is just because for any two functions $f_1(\mathbf{x})$ and $f_2(\mathbf{x})$ that only differ by a constant $c$, the linear dependence between $y$ and $\mathbf{x}$ remains the same, regardless of the value of $c$. This weakness can be addressed by offsetting the mean shift (or bias) of estimated predictor in the training data as a post-processing procedure, as has been used in HSIC loss~\cite{greenfeld2020robust}. That is, given $n$ training samples, suppose the trained network is $f_\theta$, the bias $b$ can be simply estimated by:
\begin{equation}
    b = \frac{1}{n} \sum_{i=1}^n \left[y_i - f_\theta(x_i) \right].
\end{equation}
Finally, the bias-adjusted model $f$ can be represented as: $f = f_\theta(x) + b$. 

\subsection{Differentiability of $\sqrt{J_{vN}(X:Y)}$}
\label{Differentiability}
Again, we consider the argument of the square root, i.e., $J_{vN}(X:Y)$. By definition, we have:
\begin{equation}
    D_{vN}(X||Y)=\Tr (X\log{X}-X\log{Y}-X+Y),
\end{equation}
and
\begin{equation}
    J_{vN}(X:Y)=\frac{1}{2}\left(D_{vN}(X||Y)+D_{vN}(Y||X)\right)=\frac{1}{2}\Tr{\left(\left(X-Y\right)\left(\log{X}-\log{Y}\right)\right)}.
\end{equation}
We thus have~\cite[Chapter~6]{nielsen2013matrix}:
\begin{equation}
    \frac{\partial D_{vN}(X||Y)}{\partial X}=\log{X}-\log{Y},
\end{equation}
and 
\begin{equation}
    \frac{\partial D_{vN}(X||Y)}{\partial Y}=-XY^{-1}+I,
\end{equation}
where $I$ denotes an identity matrix with the same size as $X$. Therefore,
\begin{equation}
    \frac{\partial J_{vN}(X:Y)}{\partial X}=\frac{1}{2}\left(\log{X}-\log{Y}-YX^{-1}+I\right).
\end{equation}
Since $J_{vN}(X:Y)$ is symmetric, the same applies for $\frac{\partial J_{vN}(X:Y)}{\partial Y}$ with exchanged roles between $X$ and $Y$.

In practice, taking the gradient of $J_{vN}(X:Y)$ is simple with any automatic differentiation software, like PyTorch \cite{paszke2019pytorch} or Tensorflow~\cite{abadi2016tensorflow}. We use PyTorch in this work. 


\subsection{Triangle Inequality of $\sqrt{J_{vN}(\sigma_{\mathbf{x},f(\mathbf{x})}:\sigma_{\mathbf{x},\hat{f}(\mathbf{x})})}$}

In fact, for three symmetric positive definite (SPD) matrices $X,Y,Z$ of the same size, we have:
\begin{equation}
\sqrt{J_{vN}(X:Y)}\leq \sqrt{J_{vN}(X:Z)} + \sqrt{J_{vN}(Z:Y)}, 
\end{equation}
proof in~\cite{nielsen2009sided,taghia2019constructing}.

\subsection{Robustness of $\sqrt{J_{vN}(\sigma_{\mathbf{x},f(\mathbf{x})}:\sigma_{\mathbf{x},\hat{f}(\mathbf{x})})}$}
\label{Sec:Robustness_of_J_vN}

Our loss depends on the covariance or linear correlation between $y$ and each dimension of $\mathbf{x}$, which makes our loss more robust than MSE and CE. This is again because the dependence between $y$ and $\mathbf{x}$ stays the same if the additive noise is independent to $y$ or $\mathbf{x}$ (a common assumption in signal processing and machine learning). However, MSE suffers from additive noise on $y$.
This robustness can be observed in Figure~\ref{fig:mixture_of_noise_four_datasets}, which uses four regression benchmark data sets with a mixture of Gaussian noise. 
The source and a description of the used datasets are as follows:
\begin{itemize}
    \item 2dplanes: This is an artificial dataset that was described in \cite{breiman1984classification}.
    \item cal\_housing: California housing dataset is a real dataset that is generated from the 1990 Census in California\footnote{\url{https://www.dcc.fc.up.pt/~ltorgo/Regression/cal_housing.html}} \cite{pace1997sparse}. Each sample contains a block group averaged over the individuals in that group.
    \item bank8FM: This data is generated from a simulation on the customer behavior while choosing their banks; we obtained this data from the Delve repository\footnote{\url{https://www.cs.toronto.edu/~delve/data/datasets.html}}.
    \item puma8NH: This dataset is generated by simulating the movement of a Unimation Puma 560 robot arm; it was also obtained from the Delve repository\footnote{\url{https://www.cs.toronto.edu/~delve/data/datasets.html}}. 
\end{itemize}

\begin{figure}
     \centering
     \subfloat[2dplanes\label{fig:2dplanes}]{         
         \includegraphics[width=0.45\textwidth,trim=0 40 0 0, clip]{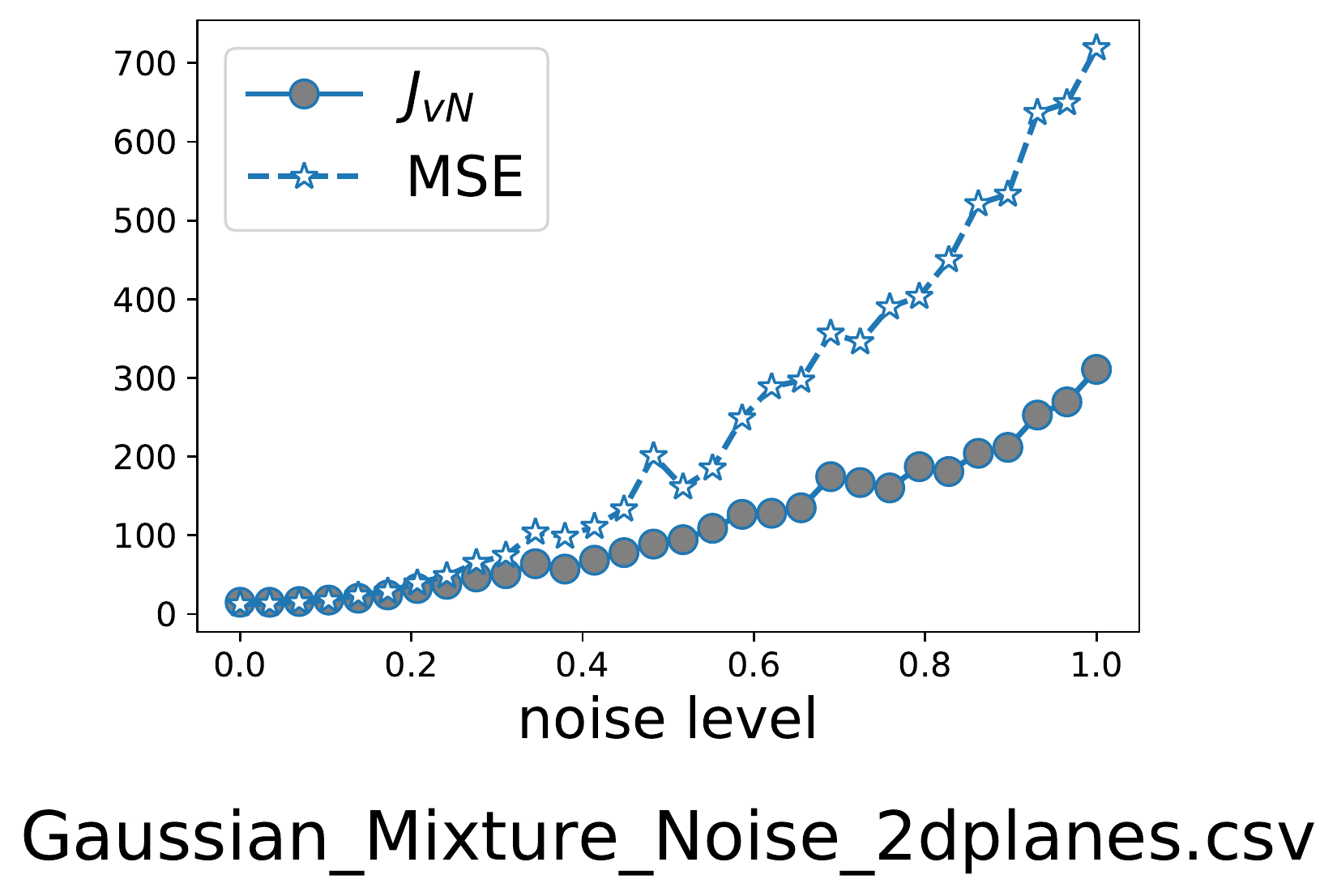}
     }
     \hfill
     \subfloat[bank8FM\label{fig:bank8FM}]{         
         \includegraphics[width=0.45\textwidth,trim=0 40 0 0, clip]{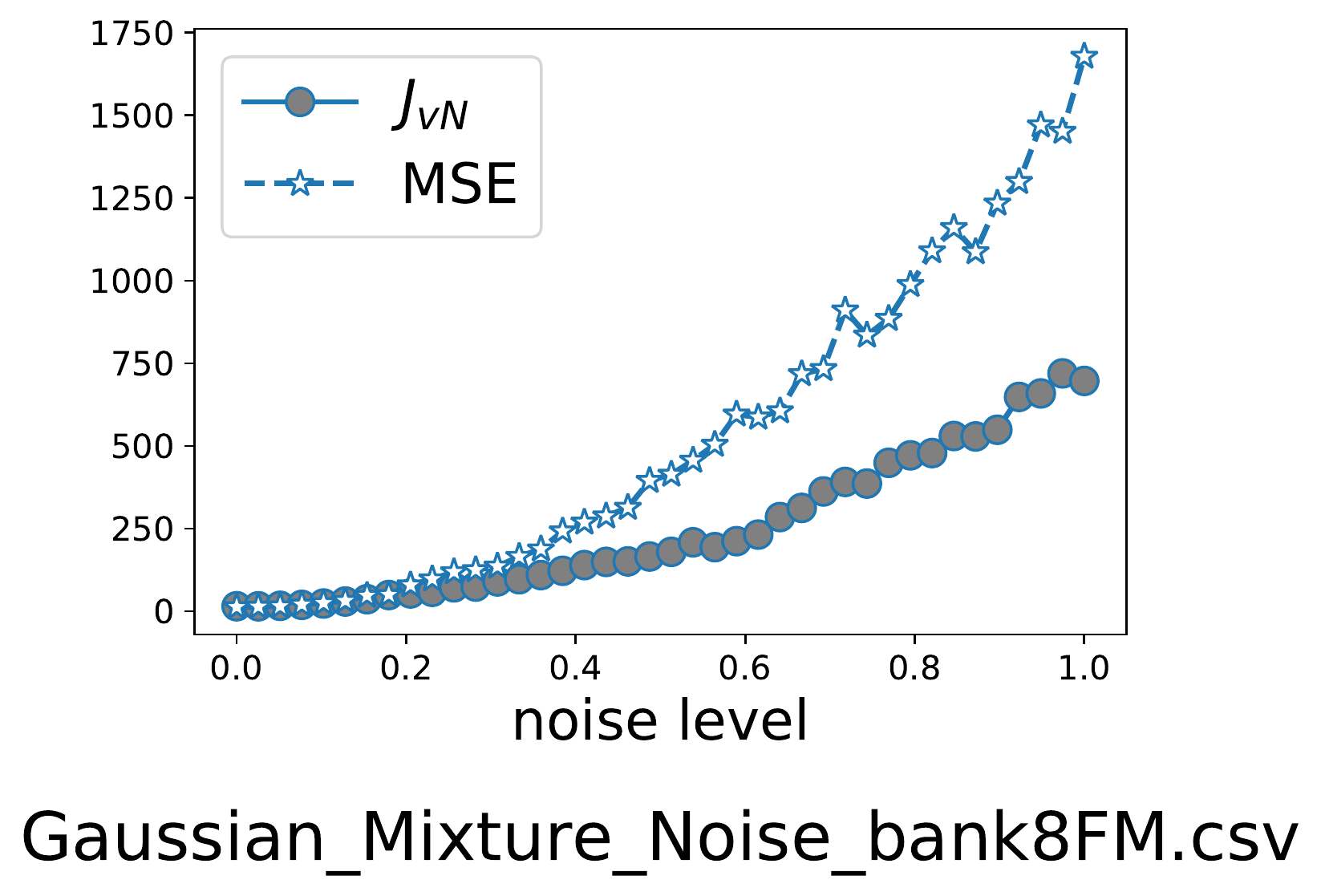}
     }
     \hfill
     \subfloat[calhousing\label{fig:calhousing}]{         
         \includegraphics[width=0.45\textwidth,trim=0 40 0 0, clip]{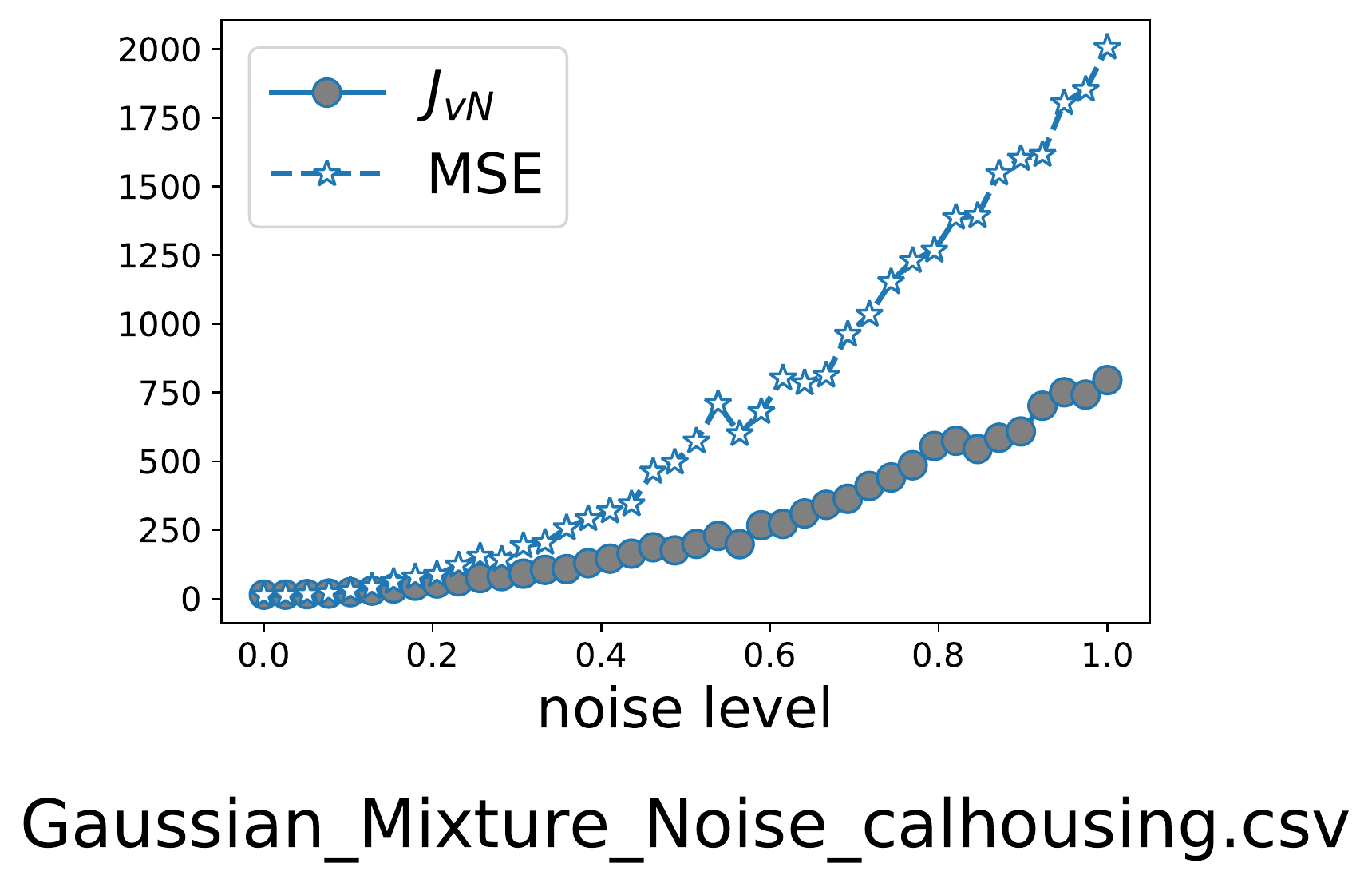}
     }
     \hfill
     \subfloat[puma8NH\label{fig:puma8NH}]{         
         \includegraphics[width=0.45\textwidth,trim=0 40 0 0, clip]{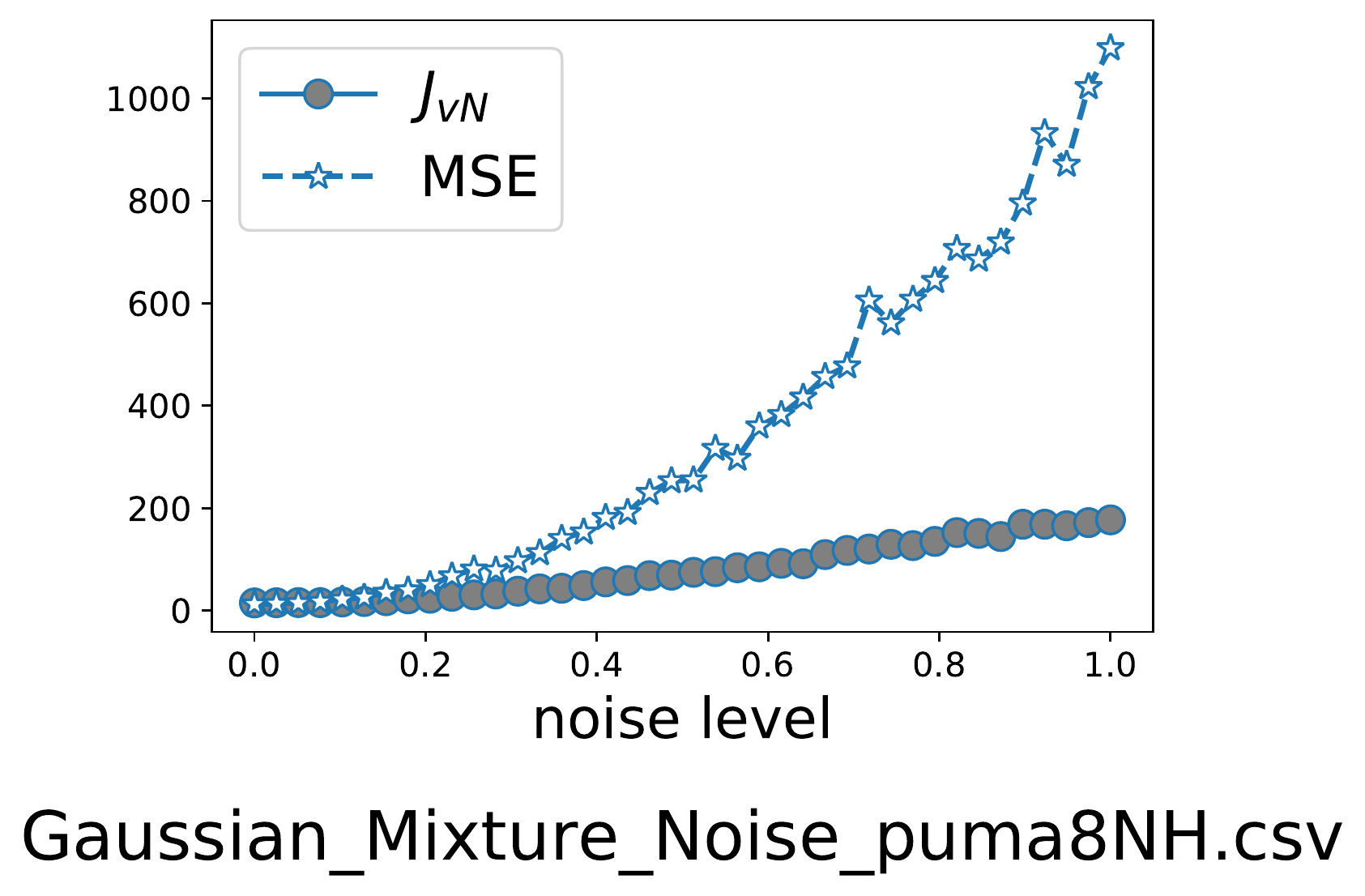}
     }
        \caption{The robustness of $\sqrt{J_{vN}}$ against MSE under MoG noises (added to the training data) on four regression benchmark data sets. (a) 2dplanes with MoG noises $\lambda[0.39\mathcal{N}(2.62,2.0)+0.37\mathcal{N}(5.98,2.1)+0.24\mathcal{N}(4.3,3.1)]$. 
(b) bank8FM with MoG noises $\lambda[0.4\mathcal{N}(3.935,4. ,)+0.38\mathcal{N}(5.693,0.979)+0.2\mathcal{N}(4.7,3.1)]$. 
(c) calhousing with MoG noises $\lambda[0.55\mathcal{N}(6.2,2.2)+0.4\mathcal{N}(4.5,3.9)+0.05\mathcal{N}(3.2,2.9)]$.(d) puma8NH with MoG noises $\lambda[0.58\mathcal{N}(4.2,0.8)+0.2\mathcal{N}(5.0,2.3)+0.22\mathcal{N}(2.1,1.1)]$. Our $\sqrt{J_{vN}}$ performs much more stable with the increase of noise level $\lambda$.}
        \label{fig:mixture_of_noise_four_datasets}
\end{figure}

\section{Convergence Behavior of the Matrix-based von Neumann Divergence}

To complement the matrix-based von Neumann divergence, we additionally provide the convergence behavior analysis of this new divergence on sample covariance matrix, which is missing in~\cite{yu2020measuring}.


First, the essence of the matrix-based divergence is to transform the problem on measuring probability distance as another problem on measuring the closeness of a few key characteristics associated with the underlying probability. For our case, we actually use the covariance matrix (i.e., the $2$nd order information) as a characterization of the underlying probability. Depending on the application, one can also use correntropy matrix to incorporate higher-order information, see~\cite{yu2020measuring}.

In this sense, the convergence behavior analysis of the matrix-based von Neumann divergence on sample covariance matrix actually includes two components: 1) how good/trustable is the covariance matrix as a complete characterization of the distribution?; and 2) how precise is the sample covariance matrix as an approximation to the ground truth covariance matrix?

For point 1, it is hard to give a bound because there are always counter-examples in which two different distributions have the same 2nd order information. If two distributions differ in the mean (i.e., 1st order information), our loss is still effective by simply offsetting the mean shift (or bias) of estimated predictor in the training data as a post-processing procedure (see Section 2.1 of supplementary material). In practice, we observed that the covariance matrix always works well. 

For point 2, we show in the following how the eigenvalues logarithmically control the convergence.

\begin{proposition}\label{proposition}

The convergence rate of the sample von Neumann divergence, $D_{vN}(\hat{\Sigma} || \hat{\Theta})$, to the true von Neumann divergence, $D_{vN}(\Sigma || \Theta)$, is controlled logarithmically by the eigenvalues of the sample covariance matrix whose distance to the true covariance matrix does not exceed $\epsilon$ with probability $1-\delta$, under the assumption of distributions with finite moments.
\end{proposition}
\begin{proof} [Proof of Proposition~\ref{proposition}]	

Vershynin~\cite{vershynin2012close} shows that for distributions of finite moments, the sample complexity is of $O(N)$ to achieve a distance $\epsilon$ between the sample and the true covariance matrix, $\hat{\Sigma}$ and $\Sigma$, of an $n$-dimensional random variable $X$. For $X$ with $q$-th moment being constant with appropriate absolute constant and 
$\lvert\lvert X \rvert\rvert_2 < K \sqrt{n}$, $\mathbb{E} \lvert \langle X,x \rangle \rvert^q \leq L^q$  for $x \in S^{n-1}$, then with probability $1-\delta$ and for some $K$ and $L$ the following holds:
\begin{equation}
\lvert\lvert  \Sigma - \hat{\Sigma} \rvert\rvert_2 \leq C_{q,K,L,\delta} (\log \log n)^n (\frac{n}{N})^{\frac{1}{2}-\frac{2}{q}}, 
\end{equation}
where $C_{q,K,L,\delta}$ depends only on $q,K,L,\delta$, $\delta>0$, and $N$ is the number of samples.
Building on this result and assuming $\epsilon =C_{q,K,L,\delta} (\log \log n)^n (\frac{n}{N})^{\frac{1}{2}-\frac{2}{q}}$, we know that $\lvert \lambda_1(\Sigma) - \lambda_1(\hat{\Sigma}) \lvert \leq \epsilon$,
where $\lambda_1(\Sigma)$ and $\lambda_1(\hat{\Sigma})$ are the largest eigenvalues of $\Sigma$ and $\hat{\Sigma}$, respectively; And $\lambda_n(\Sigma)$ and $\lambda_n(\hat{\Sigma})$ are the smallest eigenvalues.

The von Neumann divergence between between $\Sigma$ and $\Theta$ is written as:
\begin{equation}\label{eq:convergence_1}
D_{vN}(\Sigma || \Theta ) = tr(\Sigma log \Sigma - \Sigma log \Theta - \Sigma + \Theta)= \sum_i \lambda_i \log \lambda_i - \sum_{i,j} (v_i^T u_j )\lambda_i \log \theta_j -\sum_i (\lambda_i -\theta_i),
\end{equation}
where $\lambda_i$ and $v_i$ are the eigenvalues and eigenvectors for $\Sigma$, 
and $\theta_i$ and $u_i$ are the eigenvalues and eigenvectors for $\Theta$. 
Focusing only on the first and second terms of Eq.~(\ref{eq:convergence_1}), since the final term is cancelled out when computing the symmetric divergence, we can derive the following contribution of $\lambda_i$ to the convergence between $D_{vN}(\Sigma || \Theta )$ and $D_{vN}(\hat{\Sigma} || \hat{\Theta})$:
\begin{equation}\label{eq:convergence_2}
\begin{aligned}
&\lambda_i \log(\lambda_i) -c_{i,j} \lambda_i \log \theta_j - \hat{\lambda}_i \log(\hat{\lambda}_i) + \hat{c}_{i,j} \hat{\lambda}_i \log \hat{\theta}_j \\
&=\lambda_i \log(\lambda_i) - \hat{\lambda}_i \log(\hat{\lambda}_i)
-c_{i,j} \lambda_i \log \theta_j + \hat{c}_{i,j} \hat{\lambda}_i \log \hat{\theta}_j\\
& \leq(\hat{\lambda}_i +\epsilon) \log(\hat{\lambda}_i +\epsilon) - \hat{\lambda}_i \log(\hat{\lambda}_i) -c_{i,j} \lambda_i \log \theta_j + c_{i,j} (\lambda_i +\epsilon) \log(\theta_j +\epsilon) \\
& \leq \epsilon \log(\hat{\lambda}_i +\epsilon) + \hat{\lambda}_i \log( \frac{\hat{\lambda}_i+\epsilon}{\hat{\lambda}_i}) + c_{i,j} \lambda_i \log \frac{\theta_j+\epsilon}{\theta_j} + \epsilon c_{i,j} \log(\theta_j +\epsilon) \\
& \leq \epsilon \log(\hat{\lambda}_i +\epsilon) + \hat{\lambda}_i \log( 1+\epsilon) + c_{i,j} \lambda_i \log(1+\epsilon) + \epsilon c_{i,j} \log(\theta_j +\epsilon),
\end{aligned}
\end{equation}
where $c_{i,j}=v_i^T u_j$, $\hat{c}_{i,j}=\hat{v}_i^T \hat{u}_j$, and assuming that $c_{i,j} \le \hat{c}_{i,j}$ (due to symmetry, the computation would still be valid when the $c_{i,j} < \hat{c}_{i,j}$).

It is clear that Eq.~(\ref{eq:convergence_2}) is dominated by $\epsilon \sum_i \log(\hat{\lambda_i} +\epsilon) +\sum_{i,j} \epsilon c_{i,j} \log(\theta_j +\epsilon)$. Hence, the convergence bound is controlled logarithmically by the eigenvalues of the studied matrix scaled by $\epsilon$. 
\end{proof}

\section{Multi-Source Domain Adaptation with Matrix-based Discrepancy Distance}
\label{Multi-Source Domain Adaptation with Matrix-based Discrepancy Distance}
\textit{Reminder: Theorem~\ref{theorem:Q_{B}_bound}.
Given a set of $K$ source domains $S=\{D_{s_1},\dots,D_{s_K}\}$ and denote the ground truth mapping function in $D_{s_i}$ as $f_{s_i}$. Let us attribute weight $w_i$ to source $D_{s_i}$ (subject to $\sum_{i=1}^{K}{w_i=1}$) and generate a weighted source domain $D_\alpha$, such that the source distribution $P_\alpha=\sum_{i=1}^{K}{w_i P_{s_i}}$ and the mapping function\\
$f_\alpha: x\rightarrow \left(\sum_{i=1}^{K}w_i P_{s_i}(x)f_{s_i}(x)\right)/\left(\sum_{i=1}^{K}w_i P_{s_i}(x)\right)$. For any hypothesis $h\in \mathcal{H}$, the square root of $J_{vN}$ on the target domain $D_t$ is bound in the following way:
	\begin{equation}
	\sqrt{J_{vN}(\sigma^t_{x,h(x)}:\sigma^t_{x,f_t(x)})} \leq \sum_{i=1}^K w_i \left( \sqrt{J_{vN}(\sigma^{s_i}_{x,h(x)}:\sigma^{s_i}_{x,f_{s_i}(x)})} \right) 
	+ D_{\text{M-disc}}(P_t,P_{\alpha};h) + \eta_{Q}(f_\alpha,f_t),
	\end{equation}
where $\eta_{Q}(f_\alpha,f_t)= \min_{h^{*} \in \mathcal{H}} \sqrt{J_{vN}(\sigma^t_{x,h^{*}(x)}:\sigma^t_{x,f_t(x)})} +  \sqrt{J_{vN}(\sigma^\alpha_{x,h^{*}(x)}:\sigma^\alpha_{x,f_{\alpha}(x)})}$ is the minimum joint empirical losses on source $D_\alpha$ and the target $D_t$, achieved by an optimal hypothesis $h^{*}$.}

\begin{proof} [Proof of Theorem~\ref{theorem:Q_{B}_bound}]	
	
	For the weighted source $D_\alpha$ with distribution $P_\alpha$ and true mapping function $f_\alpha$, the following bound holds for each $h \in \mathcal{H}$:
	\begin{align} 
	\sqrt{J_{vN}(\sigma^t_{x,h(x)}\|\sigma^t_{x,f_t(x)})} &\leq \sqrt{J_{vN}(\sigma^\alpha_{x,h(x)}\|\sigma^\alpha_{x,f_\alpha(x)})}+ \left|\sqrt{J_{vN}(\sigma^t_{x,h(x)}\|\sigma^t_{x,f_t(x)})} - \sqrt{J_{vN}(\sigma^\alpha_{x,h(x)}\|\sigma^\alpha_{x,f_\alpha(x)})}\right| \label{eq:singlesource1} \\
	&\leq \sqrt{J_{vN}(\sigma^\alpha_{x,h(x)}\|\sigma^\alpha_{x,f_\alpha(x)})}+  \color{red}{\left|\sqrt{J_{vN}(\sigma^t_{x,h(x)}\|\sigma^t_{x,h^*(x)})} - \sqrt{J_{vN}(\sigma^t_{x,h(x)}\|\sigma^t_{x,f_t(x)})}\right|} \nonumber \\
	& + \color{green}{\left|\sqrt{J_{vN}(\sigma^\alpha_{x,h(x)}\|\sigma^\alpha_{x,h^*(x)})} - \sqrt{J_{vN}(\sigma^\alpha_{x,h(x)}\|\sigma^\alpha_{x,f_\alpha(x)})}\right|} \nonumber \\
	& + \color{blue}{\left|\sqrt{J_{vN}(\sigma^t_{x,h(x)}\|\sigma^t_{x,h^*(x)})} - \sqrt{J_{vN}(\sigma^\alpha_{x,h(x)}\|\sigma^\alpha_{x,h^*(x)})}\right|} \label{eq:singlesource2}  \\
	&\leq \sqrt{J_{vN}(\sigma^\alpha_{x,h(x)}\|\sigma^\alpha_{x,f_\alpha(x)})} +  \eta_{Q}(f_\alpha,f_t) +  D_{\text{M-disc}}(P_t,P_\alpha;h) \label{eq:singlesource3},
	\end{align}
	where $\eta_{Q}(f_\alpha,f_t)= \min_{h^{*} \in \mathcal{H}} \sqrt{J_{vN}(\sigma^t_{x,h^{*}(x)}\|\sigma^t_{x,f_t(x)})} +  \sqrt{J_{vN}(\sigma^\alpha_{x,h^{*}(x)}\|\sigma^\alpha_{x,f_{\alpha}(x)})}$ is the minimum joint empirical losses on source $D_\alpha$ and the target $D_t$, achieved by an optimal hypothesis $h^{*}$.
	
	Inequality~(\ref{eq:singlesource1}) holds since $\sqrt{J_{vN}}$ is always non-negative. Inequality~(\ref{eq:singlesource3}) follows from the triangular inequality of $\sqrt{J_{vN}}$ (i.e., $\textcolor{red}{\left|\sqrt{J_{vN}(\sigma^t_{x,h(x)}\|\sigma^t_{x,h^*(x)})} - \sqrt{J_{vN}(\sigma^t_{x,h(x)}\|\sigma^t_{x,f_t(x)})}\right|} \leq \sqrt{J_{vN}(\sigma^t_{x,h^{*}(x)}\|\sigma^t_{x,f_t(x)})}$ and \\ $\textcolor{green}{\left|\sqrt{J_{vN}(\sigma^\alpha_{x,h(x)}\|\sigma^\alpha_{x,h^*(x)})} - \sqrt{J_{vN}(\sigma^\alpha_{x,h(x)}\|\sigma^\alpha_{x,f_\alpha(x)})}\right|} \leq \sqrt{J_{vN}(\sigma^\alpha_{x,h^{*}(x)}\|\sigma^\alpha_{x,f_\alpha(x)})}$) and \\ $\textcolor{blue}{\left|\sqrt{J_{vN}(\sigma^t_{x,h(x)}\|\sigma^t_{x,h^*(x)})} - \sqrt{J_{vN}(\sigma^\alpha_{x,h(x)}\|\sigma^\alpha_{x,h^*(x)})}\right|} \leq D_{\text{M-disc}}(P_t,P_\alpha;h)$ by definition of matrix-based discrepancy distance.
	
	On the other hand, by definition we have: 
	\begin{equation}
	f_\alpha(x) = \sum_{i=1}^K w_i f_{s_i}(x), \text{s.t.}, \sum_i^K w_i=1.
	\end{equation}
	
	Therefore, for each $h\in \mathcal{H}$, we have:
	\begin{equation}
	    f_\alpha(x) - h(x) = \sum_{i=1}^K w_i \left(f_{s_i}(x)-h(x)\right),
	\end{equation}
	hence, the prediction residual on domain $D_\alpha$ is also a weighted combination of the prediction residual from each source domain $D_{s_i}$. If one evaluates prediction residual with a convex function $\epsilon$, such as the mean absolute error (MAE) loss, the mean squared error (MSE) loss or the loss defined by von Neumann divergence~\cite{bauschke2001joint,nielsen2013matrix}, it follows that:
	\begin{equation}
	    \epsilon_\alpha(f_\alpha(x),h(x)) \leq \sum_{i=1}^K w_i \epsilon_i\left(f_{s_i}(x),h(x)\right).
	\end{equation}
	
	In our case, it suggests that:
	\begin{equation}
	    \sqrt{J_{vN}(\sigma^\alpha_{x,h(x)}\|\sigma^\alpha_{x,f_\alpha(x)})} \leq \sum_{i=1}^K w_i \left( \sqrt{J_{vN}(\sigma^{s_i}_{x,h(x)}:\sigma^{s_i}_{x,f_{s_i}(x)})} \right)
	    \label{eq:convex_loss}
	\end{equation}
	
	Combining inequalities~(\ref{eq:singlesource3}) and (\ref{eq:convex_loss}), we conclude the proof.
\end{proof}

\section{Further Note on EWC}
\label{On the Relation between RSP and EWC}

\subsection{Elastic Weight Consolidation}
	\label{EWC}
	Kirkpatrick et. al argue, in EWC \cite{kirkpatrick2017overcoming}, from a Bayesian point of view that the log-posterior probability of the parametrization $\theta$, after observing two consequentive tasks $T_A$ and $T_B$, can be decomposed into the log-likelihood of the task $T_B$ given the current network and the log-prior $\log p(\theta|T_A)$ (which is the same as the log-posterior given the previous task $T_A$), i.e., 
	\[
	\log p(\theta|T_A,T_B)  = \log p(T_B|\theta) + \log p(\theta|T_A) - \log p(T_B|T_A) .
	\]
	Using Laplace approximation, the log-posterior distribution $\log p(\theta|T_A)$ is approximated by a Gaussian distribution with mean $\theta_{A,i}^*$, and the inverse of the Hessian of the negative log-likelihood $-\log p(\theta|T_A)$ gives the variance. This is further simplified by taking the precision matrix as the diagonal Fisher information matrix $F_\theta$. As a result, the loss function is re-written as $
	\mathcal{L}(\theta) = \mathcal{L}_B(\theta) + \sum_i \frac{\lambda}{2} \mathcal{F}_{\theta_i} (\theta_i - \theta_{A,i}^*)^2$, with $\mathcal{L}_B(\theta)$ being the loss for task $T_B$, and $\lambda$ is the importance of the previous task.
\subsection{Special Relation to EWC and Fisher Information}	
\label{Special Relation to EWC and Fisher Information}	
Chaudhry et. al \cite{chaudhry2018riemannian} show that the KL-divergence $D_{KL}(p_{\theta}(y|x)||p_{\theta + \Delta\theta}(y|x))$ between conditional likelihoods of two neural networks parametrized by $\theta$ and $\theta + \Delta\theta$ can be approximated as $D_{KL}(p_{\theta}(y|x)||p_{\theta + \Delta\theta}(y|x)) \approx  \frac{1}{2} \Delta{\theta}^T \mathcal{F}_\theta \Delta\theta$ where $F_\theta$ is the Fisher information matrix at $\theta$, assuming that $\Delta\theta \to 0$, see the proof in Appendix A1 of \cite{chaudhry2018riemannian}. Since it is infeasible to compute $\mathcal{F}_\theta$ when the number of parameters is in the order of millions, parameters are assumed to be independent and only the diagonal of $F_\theta$ is computed, as a result, the divergence becomes $D_{KL}(p_{\theta}(y|x)||p_{\theta + \Delta\theta}(y|x)) \approx \sum_{\theta_i}  \frac{1}{2} \mathcal{F}_{\theta_i} \Delta\theta_i^2$ which collides with the regularization term of EWC, i.e., the second term in
\begin{align}
	\mathcal{L}(\theta) = \mathcal{L}_B(\theta) + \sum_i \frac{\lambda}{2} \mathcal{F}_{\theta_i} (\theta_i - \theta_{A,i}^*)^2	. 
	\label{eq:EWC_supp}
	\end{align}

\section{Illustrations and Complexity Analysis}
\subsection{Multi-Source Domain Adaptation with Matrix-Based Discrepancy Distance}
Figure~\ref{fig:Illustrations_MSDA} depicts an illustration of our method MDD. $X_{s_i}$ and $Y_{s_i}$ are the input samples and the ground truth from the source domain $s_i$; $X_t$ holds the input samples of the target domain without labels. Passing through the feature extractor layers $f_\theta$, the representations $f_\theta(X_{s_i})$ and $f_\theta(X_t)$ are produced.
While the hypothesis $h$ is being trained to be a good predictor (reducing weighted source risk $L_s$), the hypothesis $h^\prime$ tries to increase the matrix-based discrepancy distance between the target distribution and the weighted combination of source domains (i.e., $D_{M-disc}$).
The pseudo-code, illustrating the loss computation in the forward propagation and the parameters' update in the backward propagation, is presented in Algorithm~\ref{algo:Pseudo_algorithmMDD}.

\begin{figure}[h]
 \centering
\includegraphics[width=0.7\linewidth,trim = 150 140 170 0,clip]{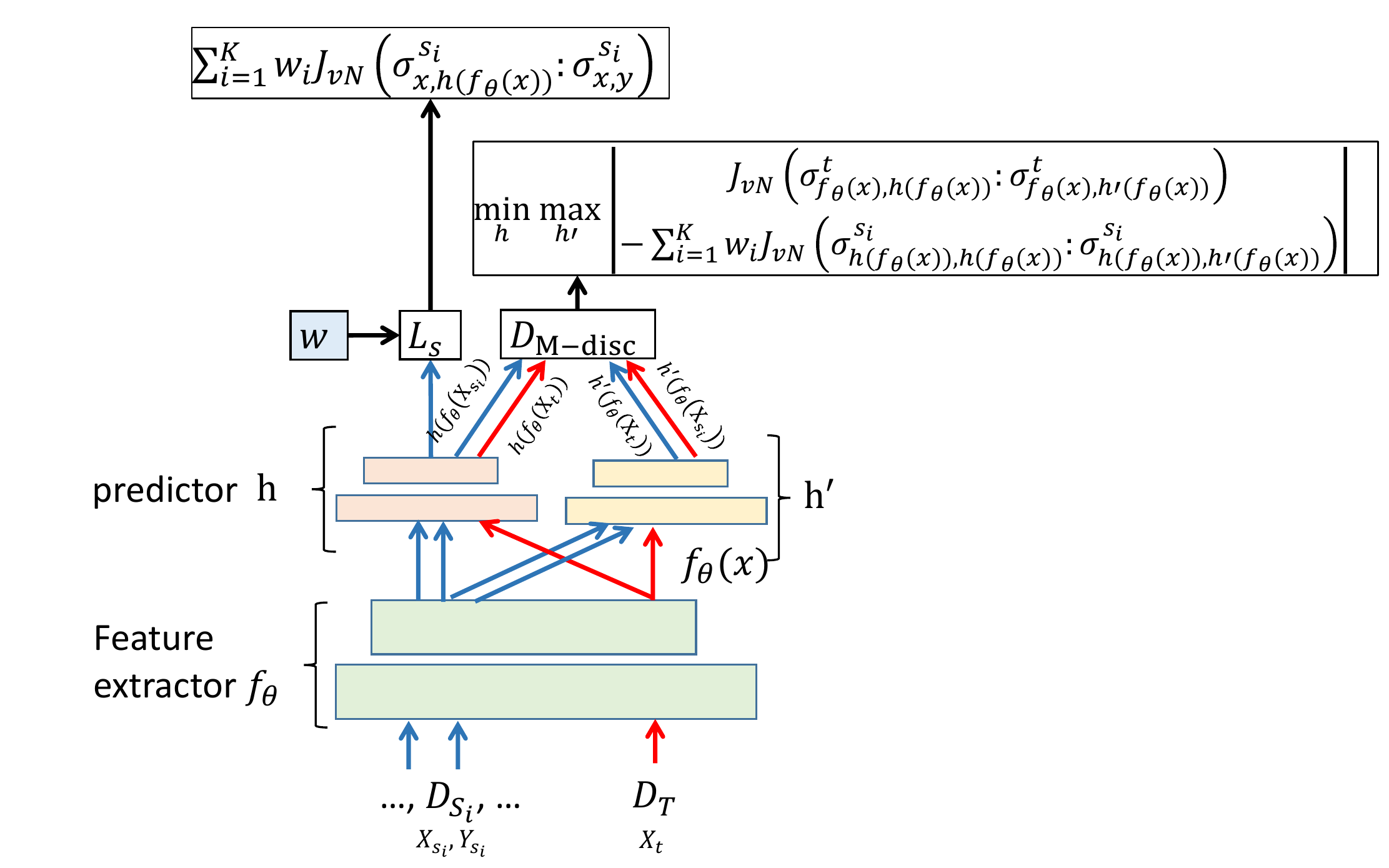}
\caption{An illustration of how the Jeffery von Neumann divergence is employed in our multi-source domain adaptation method, MDD. The objective includes two terms: 1) minimization of the weighted risk from all source domains, i.e., $L_s$; and 2) minimization of the matrix-based discrepacny distance $D_{M-disc}$ between target distribution and the weighted source distribution.}
\label{fig:Illustrations_MSDA}
\end{figure}

\begin{algorithm}[h]
\caption{Pseudo algorithm for \textbf{MDD}}
\label{algo:Pseudo_algorithmMDD}
 $h$: the predictor, $h^{\prime}$: the adversarial hypothesis\\
 $f_\theta$: the feature extractor, ${\eta}$: the learning rate, $K$: number of sources\\
 initialize $w_s$= $\frac{1}{K}$\\
\For{$i$ = $\rm{1}$...$epochs$}
{
$\textbf{Forward propagation}$\\
$e_j = w_j \sqrt{J_{vN}(\sigma^{s_j}_{x,h(f_\theta(x))}:\sigma^{s_j}_{x,y})}$ for each source $s_j$, $j \in \{1,\dots,K\}$\\
$D_{\text{M-disc}}(P_t,P_{\alpha};h) = \Big|\sqrt{J_{vN}(\sigma^t_{f_\theta(\mathbf{x}),h(f_\theta(\mathbf{x}))}:\sigma^t_{f_\theta(\mathbf{x}),h^\prime(f_\theta(\mathbf{x}))})} -\sum_{j=1}^{K}    
    w_k \sqrt{J_{vN}(\sigma^{s_j}_{f_\theta(\mathbf{x}),h(f_\theta(x))}:\sigma^{s_j}_{f_\theta(\mathbf{x}),h^\prime(f_\theta(x))})}\Big|$  \quad\quad\quad\quad\quad\quad\quad $\triangleright (\ast)$\\

$\textbf{Backward propagation}$\\

$h^{(i+1)} = h^{(i)}-\eta$($\sum_{j=1}^{K}w_j^{(i)}$$\nabla_h e_j)$ \\

$h^{\prime(i+1)} = h^{\prime(i)}+\eta(\sum_{j=1}^{K}w_j^{(i)}\nabla_{h^\prime}D_{\text{M-disc}}(P_t,P_{\alpha};h)$\\

$\theta^{(i+1)} =\theta^{(i)}-\eta(\sum_{j=1}^{K}w_j^{(i)}\nabla_{\theta}e_j+{\nabla_{\theta}} D_{\text{M-disc}}(P_t,P_{\alpha};h)$\\

$w_j^{(i+1)} = w_j^{(i)}-\eta(\nabla_{w_j}
D_{\text{M-disc}}(P_t,P_{\alpha};h)$, \quad\quad$j \in \{1,\dots,K\}$\\
$w^{(i+1)} = w^{(i+1)}/\vert\vert w^{(i+1)}\vert\vert_1$
}

$(\ast)$ $D_\alpha$ is the weighted source domain assuming that the source distribution is $P_\alpha=\sum_{j=1}^{K}{w_j P_{s_j}}$ and the mapping function is $f_\alpha: x\rightarrow \left(\sum_{j=1}^{K}w_j P_{s_j}(x)f_{s_j}(x)\right)/\left(\sum_{j=1}^{K}w_j P_{s_j}(x)\right)$. 
\end{algorithm}

The complexity for computing the von Neumann divergence $J_{vN}(\sigma^s_{f_\theta(\mathbf{x}),h(f_\theta(\mathbf{x}))}:\sigma^s_{f_\theta(\mathbf{x}),h^\prime(f_\theta(\mathbf{x}))})$ on the domain $s$ and the two hypotheses $h,h^\prime \in \mathcal{H}$ constitutes the following: \textit{(i)} computing the covariance matrix $\sigma^s_{f_\theta(\mathbf{x}),h(f_\theta(\mathbf{x}))}$ takes $O(N(d+1)^2)$ where $d$ is the size of the final layer of the feature extractor $f_\theta$ and $N$ is the batch size. \textit{(ii)} computing $J_{vN}$ on two covariance matrices from $\mathcal{R}^{(d+1)\times(d+1) }$ requires the eigenvalue decomposition which is 
$(O(d+1)^3)$. Hence, the final complexity for a batch is $O((d+1)^3+ N(d+1)^2)$. 
Notice that this complexity is independent of the dimensionality of the data and is only controlled by the dimensionality of the extracted features.


\subsection{Continual Learning by Representation Similarity Penalty}

\begin{figure}[ht]
 \centering
\includegraphics[width=0.7\linewidth,trim = 150 190 440 150,clip]{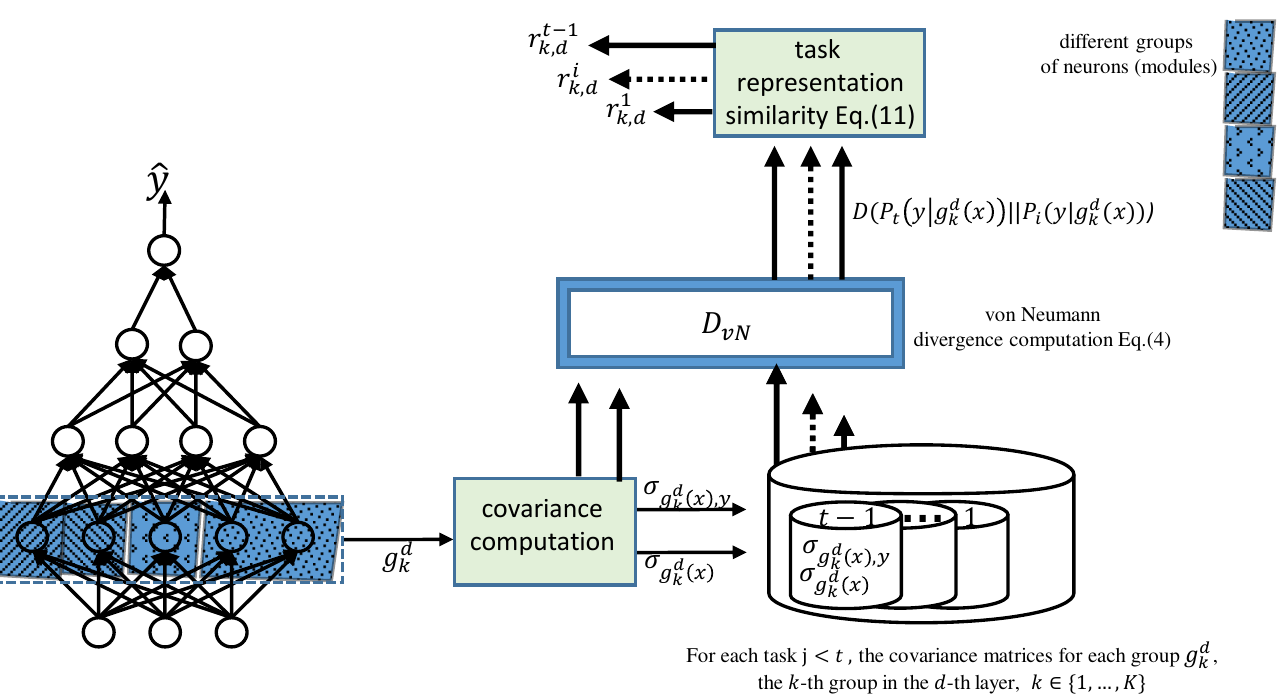}
\caption{An illustration of how the von Neumann divergence is employed in our continual learning method, RSP. The illustration depicts how the relatedness in Eq.(11) between the current task $T_t$ and each previous task $T_j$ ($j<t$) is computed, for each identified group of neurons $g_k^d$ (module) for each layer $d$, where $k\in\{1,\dots,K\}$.}
\label{fig:Illustrations_CL}
\end{figure}

Figure~\ref{fig:Illustrations_CL} depicts an illustration of our continual learning method RSP.
After training the initial network's parameters $\theta$ on the first task, the continual aspect takes place after training the initial network's parameters $\theta$ on the first task, $T_1$. Thereafter, the modular groups $\{g_1^d,\dots,g_{K_d}^d\}$ are formulated for each layer $d \in \{2,\dots,D-1\}$ (we employ the community detection method proposed in \cite{watanabe2018modular}).
On the first task $T_1$ and each following task $T_j$, the covariance matrices
$\sigma_{g_k^d (x)y}$ and $\sigma_{g_k^d (x)}$ are computed for each group $g_k^d$ in each layer $d$, where $k\in\{1,\dots,K\}$.
The matrices $\sigma_{g_k^d (x)y}$ and $\sigma_{g_k^d (x)}$ for task $T_j$ characterize $P_j(G_k^d (x),y)$ and $P_j(G_k^d (x))$. The computed matrices are maintained in the memory for a future use. 

For each forthcoming task $T_t$, for each group, $g_{i}^d$, we compute the discrepancy $D(P_t(y|g_k^d (x)||P_j(y|g_k^d (x)))$ between the conditional distributions of the current task $T_t$ and the previous tasks $T_j$ ($j<t$). The computed discrepancy shows how each pair of tasks is related given the respective module (group of neurons). The relatedness in Eq.(11), which is employed in the regularization term in Eq.(10), is computed based on the discrepancy.

For number of groups $G$ after modularization, the complexity becomes $O(G(\lceil \frac{d}{G} \rceil+1)^3+  G \cdot N(\lceil \frac{d}{G} \rceil+1)^2)$.

\section{Evaluation Details and Additional Experiments}
\label{Additional Experiments}	
\subsection{Additional Results and Information for Evaluating the Multi-Source Domain Adaptation with Matrix-based Discrepancy Distance}
\label{Additional information for Evaluating Multi-Source Domain Adaptation}
In our experiments, we used the following multi-source domain adaptation libraries:
\begin{itemize}
    \item MDAN: No specified license. \url{https://github.com/hanzhaoml/MDAN}
    \item ADisc-MSDA: No specified license. \url{https://github.com/GRichard513/ADisc-MSDA}
    \item DARN: MIT License. \url{https://github.com/junfengwen/DARN}
\end{itemize}

\subsubsection{Vehicle Counting}

TRaffic ANd COngestionS (TRANCOS) \cite{guerrero2015extremely} dataset is a public benchmark dataset for extremely overlapping vehicle counting. It contains images that were collected from 11 video surveillance cameras, that monitor different highways in the Madrid area. The images show traffic jam scenes with different scenarios, light conditions, and perspectives. The dataset contains a total of 1244 images and about 46700 manually annotated vehicles with a considerable grade of overlap. 

The images are of size $480 \times 640$ with three color channels. 
Each image is labeled using the dotting annotation method \cite{lempitsky2010learning} creating density images, besides, a mask, depicting the road's region of interest, is provided. The ground truth is turned into density maps by placing a Gaussian centered at each annotated point $p_i \in P$, where $P$ is the set of annotated vehicle positions for a single image. The resulting density map for each pixel $q$ is defined as $D(q) =\sum_{p_i}^{P} \mathcal {N}(q; p_i ,\sigma)$,
where $\sigma$ is a constant parameter that represents the smoothness of the Gaussian and it should roughly cover the area of the object. Following \cite{onoro2018learning}, we set $\sigma=10$. The total number of vehicles can be easily obtained by integrating over the density map defined by all pixels. Figure \ref{fig:Sample_Camera206} depicts a sample image with its density map and mask.

\begin{figure}[h]
     \subfloat[Original image\label{fig:Original1}]{         
        \includegraphics[scale=0.5]{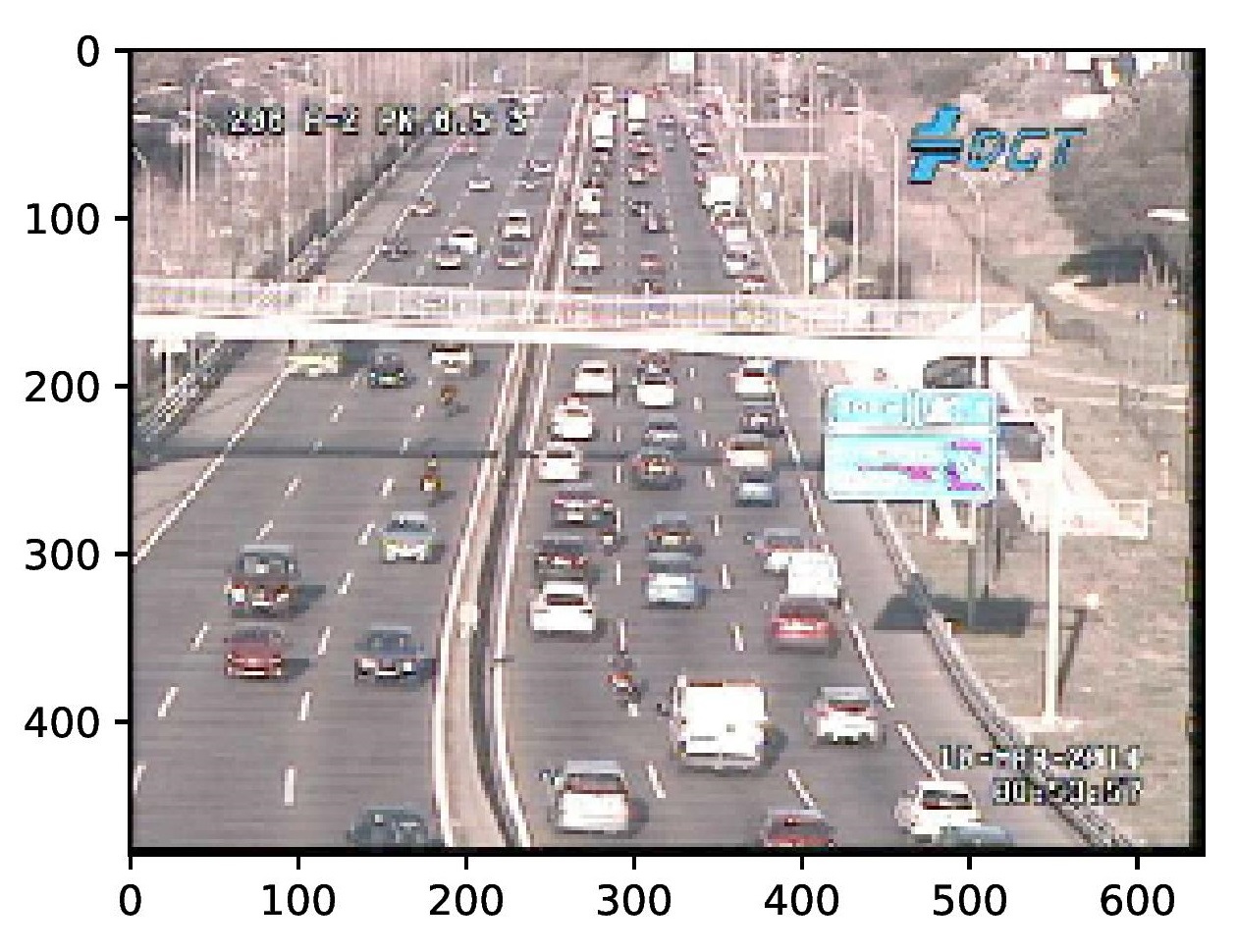}
     }
     \hfill
     \subfloat[Density map\label{fig:Original2}]{         
        \includegraphics[scale=0.5]{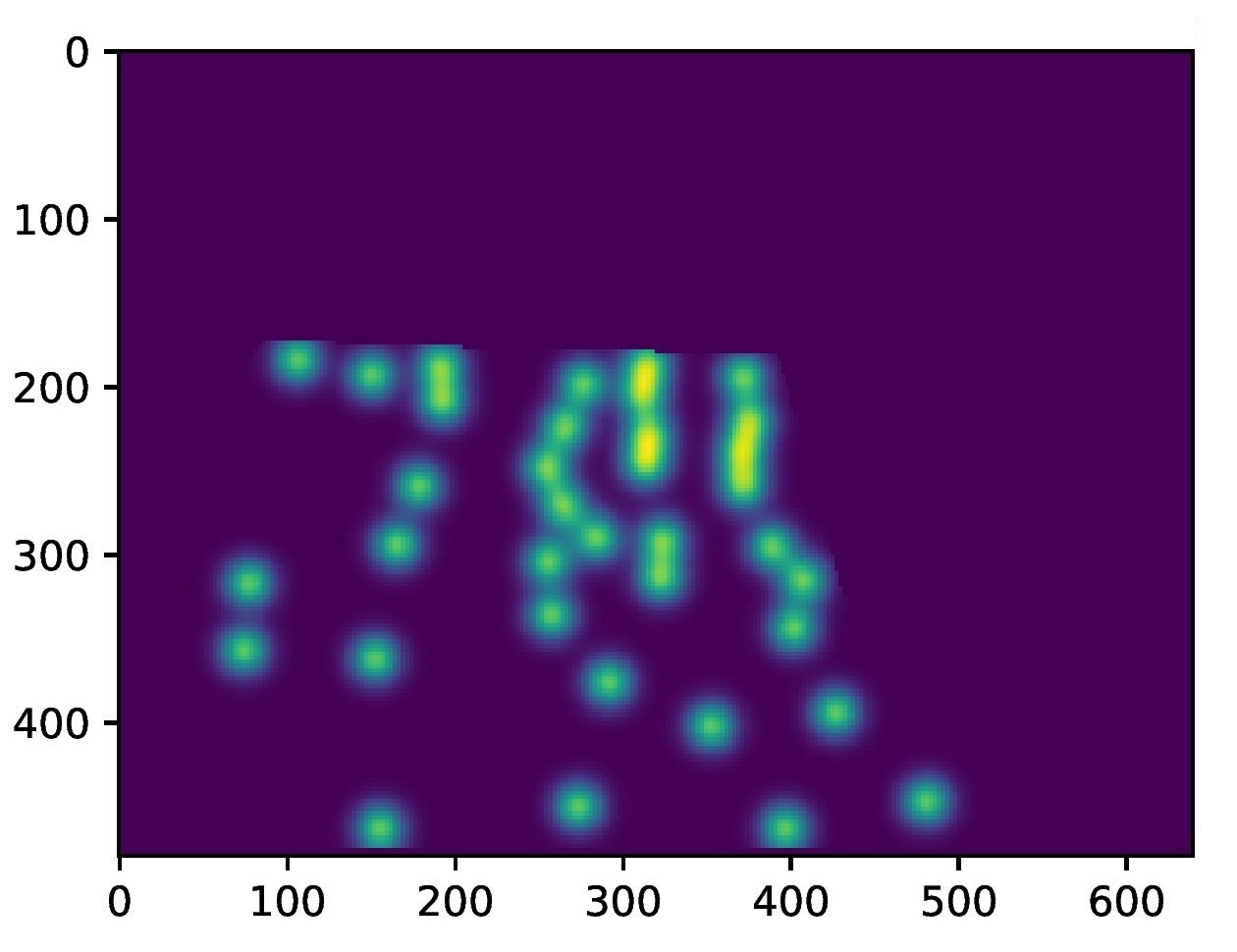}
     }
     \hfill
     \subfloat[Mask\label{fig:Original3}]{         
        \includegraphics[scale=0.5]{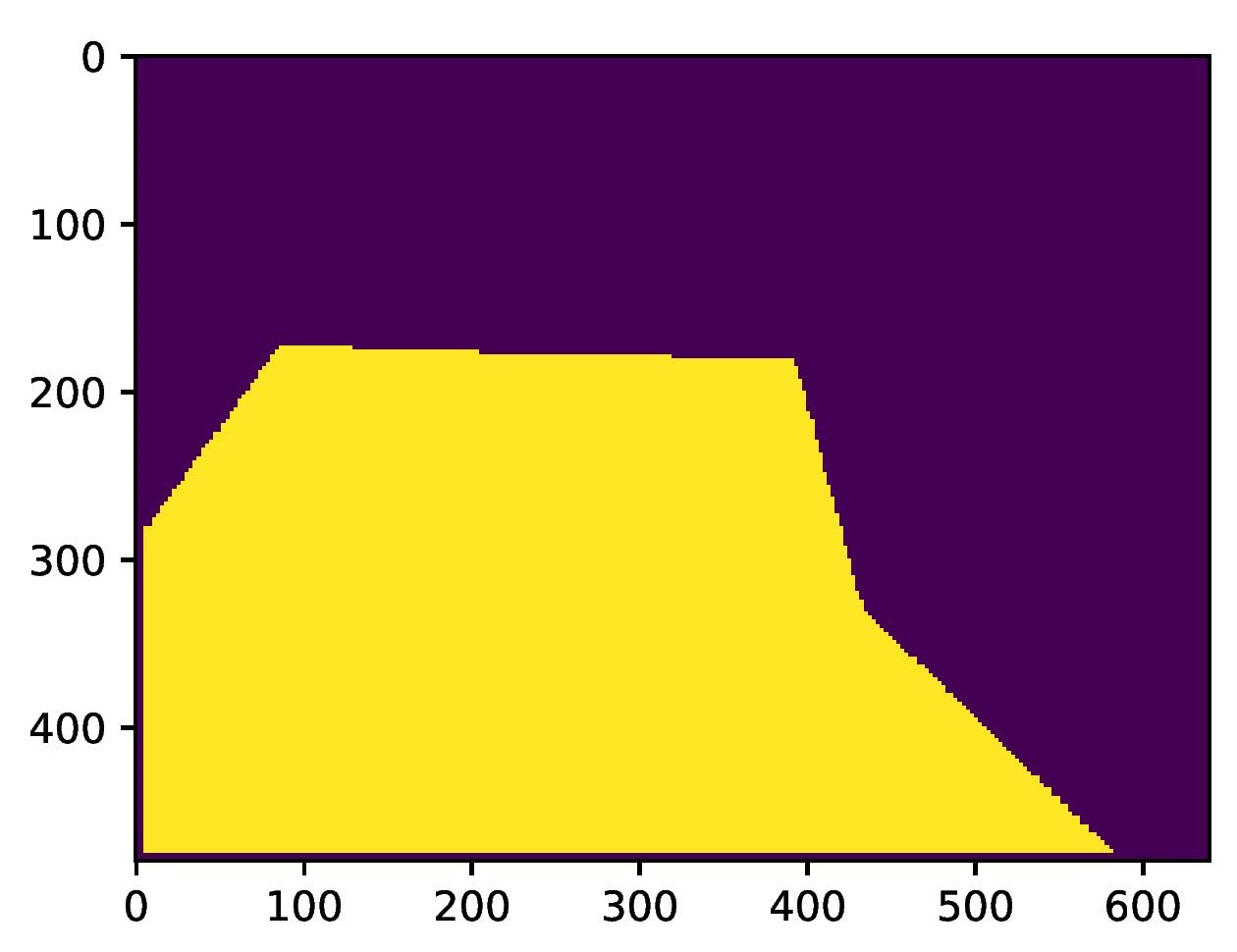}
     }
    \caption{A sample image with the corresponding density map and mask taken by camera 206.}
    \label{fig:Sample_Camera206} 
\end{figure}

\begin{figure}[h]
     \subfloat[A heatmap showing the pairwise distances between cameras.\label{fig:cameras_heatmap}]{         
        \includegraphics[scale=0.3]{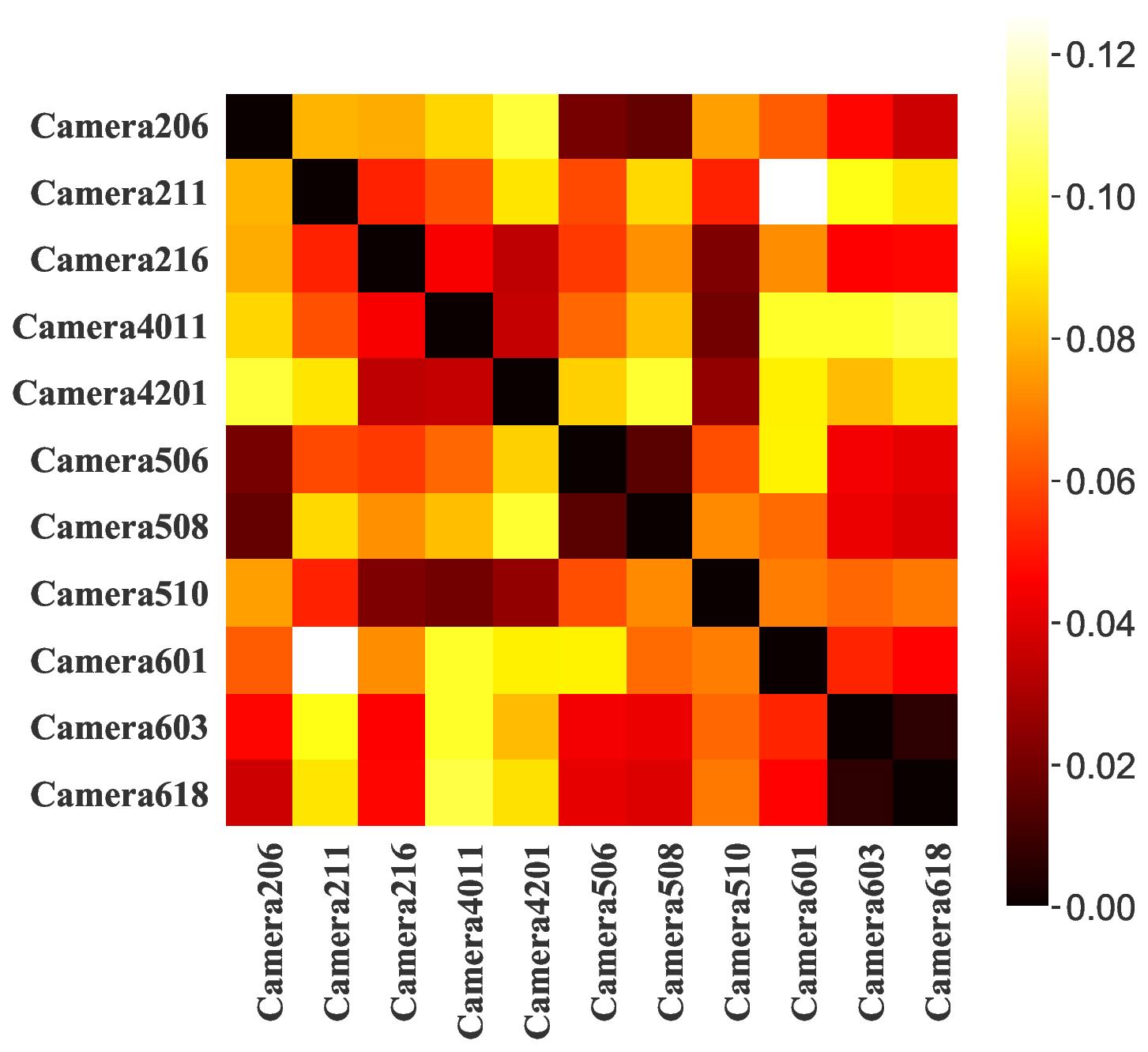}
     }
     \hfill
     \subfloat[Hierarchical clustering performed on the heatmap depicted in Figure \ref{fig:cameras_heatmap}.\\ Notice block-diagonal sub-matrices.\label{fig:clustermap}]{         
        \includegraphics[scale=0.3]{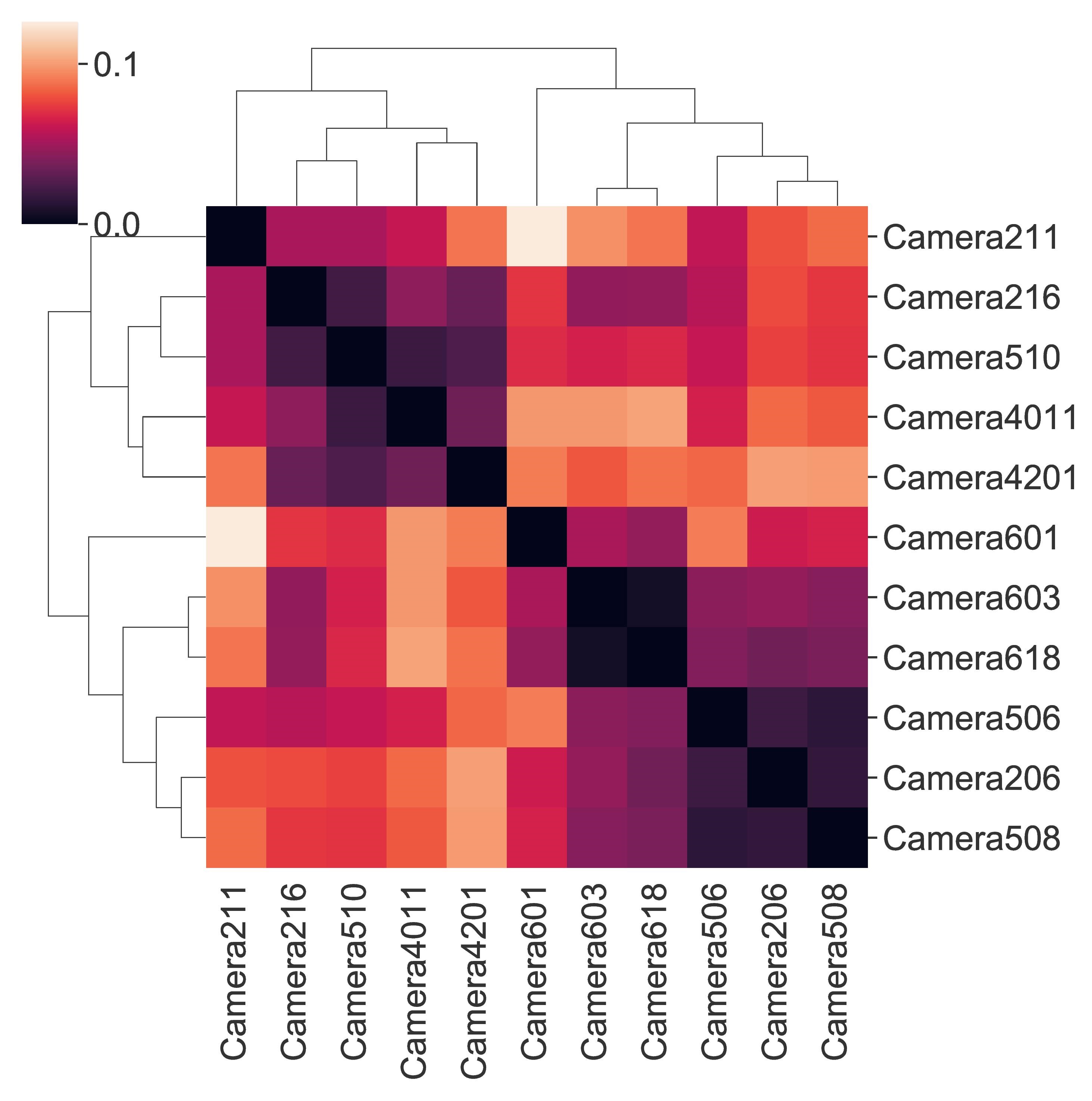}
     }     
    \caption{The heatmap and the corresponding hierarchical clustering to obtain the five domains.}
    \label{fig:cameras_heatmap_clustermap} 
\end{figure}     

\begin{figure}[h]
\centering
\includegraphics[width=0.8\linewidth]{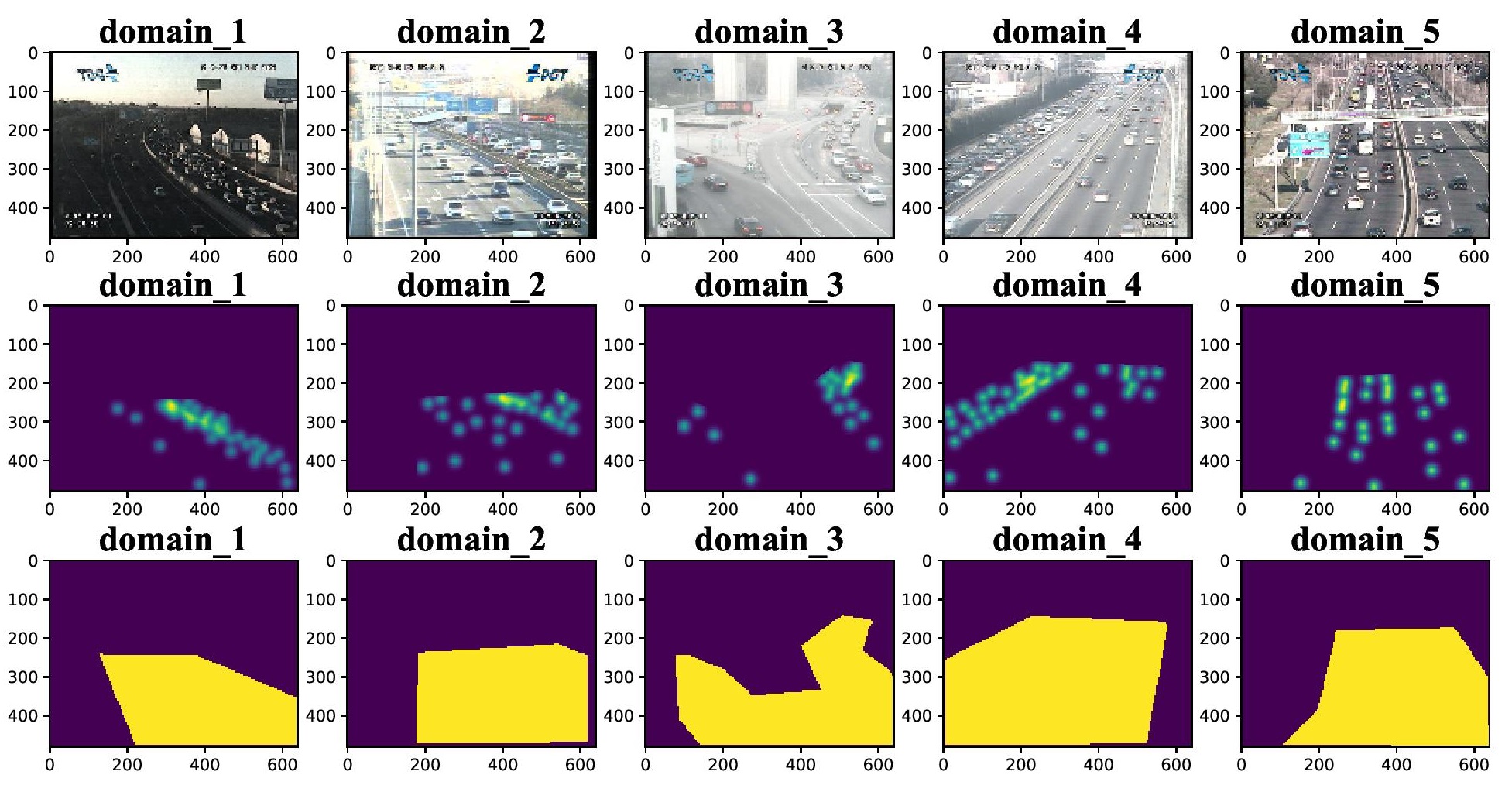}
\caption{A sample images with their corresponding density map and mask, taken from the obtained domains.}
\label{fig:Sample_All_Domains}
\end{figure}

We create source domains for the vehicle counting by grouping over the available cameras. To this end, we compute the mean mask for each camera and measure the distance between each pair of cameras (see Figure \ref{fig:cameras_heatmap}). Applying a hierarchical clustering, we identify five domains shown as block-diagonal in Figure \ref{fig:clustermap} with clear overlaps between domains. The final domains, their cameras and number of images are depicted in Table~\ref{tab:counting_domains}. Figure~\ref{fig:Sample_All_Domains} shows a sample image and its corresponding density map and  mask for each of the five discovered domains.

\begin{table}[h!]
\caption{TRANCOS: The final domains after applying hierarchical clustering on the cameras.}
\label{tab:counting_domains} 
\centering
\begin{tabular}{c|l|c}
\hline \hline
Domain & Cameras         & Num. Images \\ \hline
1      & 211, 216        & 166         \\ \hline
2      & 510, 4011, 4201 & 295         \\ \hline
3      & 601             & 274         \\ \hline
4      & 603, 618        & 252         \\ \hline
5      & 206, 506, 508   & 257         \\ \hline \hline
\end{tabular}
\end{table}

We designed an hourglass network \cite{newell2016stacked} such that the encoder is composed of 6 blocks of 2D convolution with 32 channels and $3 \times 3$ kernels, batch normalization, relu, and max-pooling operations. The decoder spatially expands and reconstructs the ground truth. It contains 6 blocks of transpose convolution layers of 32 channels and $4 \times 4$ kernels, batch normalization, and relu activation. We run each method for 30 iterations use batches of size 2 for each source. 
\begin{table*}[t]
	\caption{Performance comparison in terms of mean absolute error (MAE) over three iterations on TRANCOS data (with standard error in brackets). The best performance is marked in boldface. DARN fails to generalize on the source domains, hence, performs very poorly on the target domains.}\label{tab:Camer_counting_supp}
	  \centering
	\begin{tabular}{p{0.04\textwidth}cccc|cc|c}
		\toprule
		&AHD&DANN&AHD-&DARN&\multicolumn{2}{c|}{MDAN}&MDD\\
		&-1S&-1S&MSDA&&-Max&-Dyn& \\
		\midrule
		Dom1&46.87 \small{(12.89)}&16.19 \small{(0.42)}&57.19 \small{(22.93)}&---&32.17 \small{(7.98)}&29.35 \small{(3.96)}&\textbf{14.73} \small{(0.52)}\\
		Dom2&27.39 \small{(4.8)}&21.7 \small{(0.86)}&33.8 \small{(6.51)}&---&18.02 \small{(0.34)}&\textbf{14.34}\small{(0.24)}&15.27 \small{(0.92)}\\
		Dom3&63.69 \small{(31.62)}&28.43 \small{(5.63)}&63.27 \small{(24.77)}&---&38.5 \small{(11.77)}&26.81 \small{(4.61)}&\textbf{24.67} \small{(3.43)}\\
		Dom4&23.02 \small{(3.71)}&21.54 \small{(5.64)}&88.07\small{(52.72)} &---&19.89 \small{(3.83)}&22.86 \small{(1.04)}&\textbf{14.25} \small{(1.64)}\\
		Dom5&65.89 \small{(22.71)}&57.12 \small{(29.74)}&38.02 \small{(11.7)}&---&57.28 \small{(36.24)}&22.73\small{(4.72)}&\textbf{17.34} \small{(1.43)}\\
		\hline
	\end{tabular}
\end{table*}
To achieve a fair comparison between MDD and the other competitors, we use the encoder part network as a feature extractor, and two decoder networks one for the predictor and the other one for the discriminator. As for MDAN, each domain classifier is defined as the first four layers of the decoder followed by a linear layer. The predicted vehicle count is computed by integrating over the predicted density map after applying the ground truth mask, thereafter, the mean absolute error is computed on the predicted counts. We run each method for 30 iterations and use batches of size 2 for each domain. The reason for the small batch size is the limited computational power and memory (~16GB) our GPU has. With the employed architecture, it was possible to propagate the gradient for maximally ten images at once: $2\times4$ source domains and $2$ for the target domain.

The quantitative results are summarized in Table~\ref{tab:Camer_counting_supp}. Our MDD always achieves the smallest mean absolute error on all target domains, except for "Dom2" of the counting problem. DARN fails to generalize on the source domains of TRANCOS and, hence, performs poorly on the target domain. A close inspection of DARN's weak performance showed that the sparse source weights chosen by DARN caused the algorithm to learn from the small batch of a single random domain each time instead of exploiting all available batches of all domains. Hence, DARN failed on the source domains.

\subsubsection{The YearPredictionMSD dataset}
\label{sec:More_on_YearPredictionMSD}
The task beyond the YearPredictionMSD dataset is to predict the release year of a song based on $90$ ``timbre" features. It includes about 515k songs with release year ranging from 1922 to 2011. We obtain the version hosted at the UCI repository \cite{Dua:2019}. In order to create a multi-source problem, we try to create a set of distinctive domains. To this end, we apply $k$-means on the first 30 features and, thereafter, assign the songs of each cluster to a domain. The resulting domains are $\{$Dom$1,\dots,$Dom$5\}$.  
Figure \ref{fig:YearPredictionMSD} presents how this approach creates five distinguishable domains when shown in a t-distributed stochastic neighbor embedding (t-SNE) applied on the whole features of the dataset \cite{hinton2002stochastic}; similarly, the histograms depict how the target distributions vary considerably between the different domains.

The quantitative results on the YearPredictionMSD dataset are summarized in Table~\ref{tab:YearPredictionMSD}. Our MDD always achieves the smallest mean absolute error in all target domains, expect for "Dom3".

\begin{table*}[t]
	\caption{Performance comparison in terms of mean absolute error (MAE) over five iterations on YearPredictionMSD data (with standard error in brackets). The best performance is marked in boldface. On Dom$1$ and Dom$5$, DANN-1S fails to generalize and performs very poorly, hence, we omit its results on these domains.}\label{tab:YearPredictionMSD}
	  \centering
	\begin{tabular}{p{0.04\textwidth}cccc|cc|c}
		\toprule
		&AHD&DANN&AHD-&DARN&\multicolumn{2}{c|}{MDAN}&MDD\\
		&-1S&-1S&MSDA&&-Max&-Dyn& \\
		\midrule
		Dom1&7.10 \small{(0.07)}&---&7.04 \small{(0.07)}&7.0 \small{(0.06)}&18.1 \small{(9.2)}&16.8\small{(8.6)}&\textbf{6.91} \small{(0.08)}\\
		Dom2&8.42 \small{(0.07)}&34.9 \small{(14)}&8.28 \small{(0.02)}&8.27 \small{(0.02)}&42.8 \small{(14)}&43.4 \small{(14)}&\textbf{8.23} \small{(0.03)}\\
		Dom3&7.95 \small{(0.09)}&30.2 \small{(0.04)}&7.8 \small{(7.4)}&\textbf{7.78} \small{(0.04)}&33.4 \small{(9)}&33.8 \small{(9.4)}&7.95 \small{(0.13)}\\
		Dom4&7.74 \small{(0.04)}&22.3 \small{(7.6)}&7.61 \small{(0.04)}&7.60 \small{(0.02)} &28.5 \small{(10)}&29.9 \small{(11)}&\textbf{7.54} \small{(0.06)}\\
		Dom5&7.56 \small{(0.06)}&---&7.5 \small{(0.05)}&7.41 \small{(0.06)}&23.5 \small{(8.3)}&24.6 \small{(8)}&\textbf{7.31} \small{(0.09)}\\
		\hline
	\end{tabular}
\end{table*}

\begin{figure}[h]
     \subfloat[]{         
        \includegraphics[width=.49\linewidth]{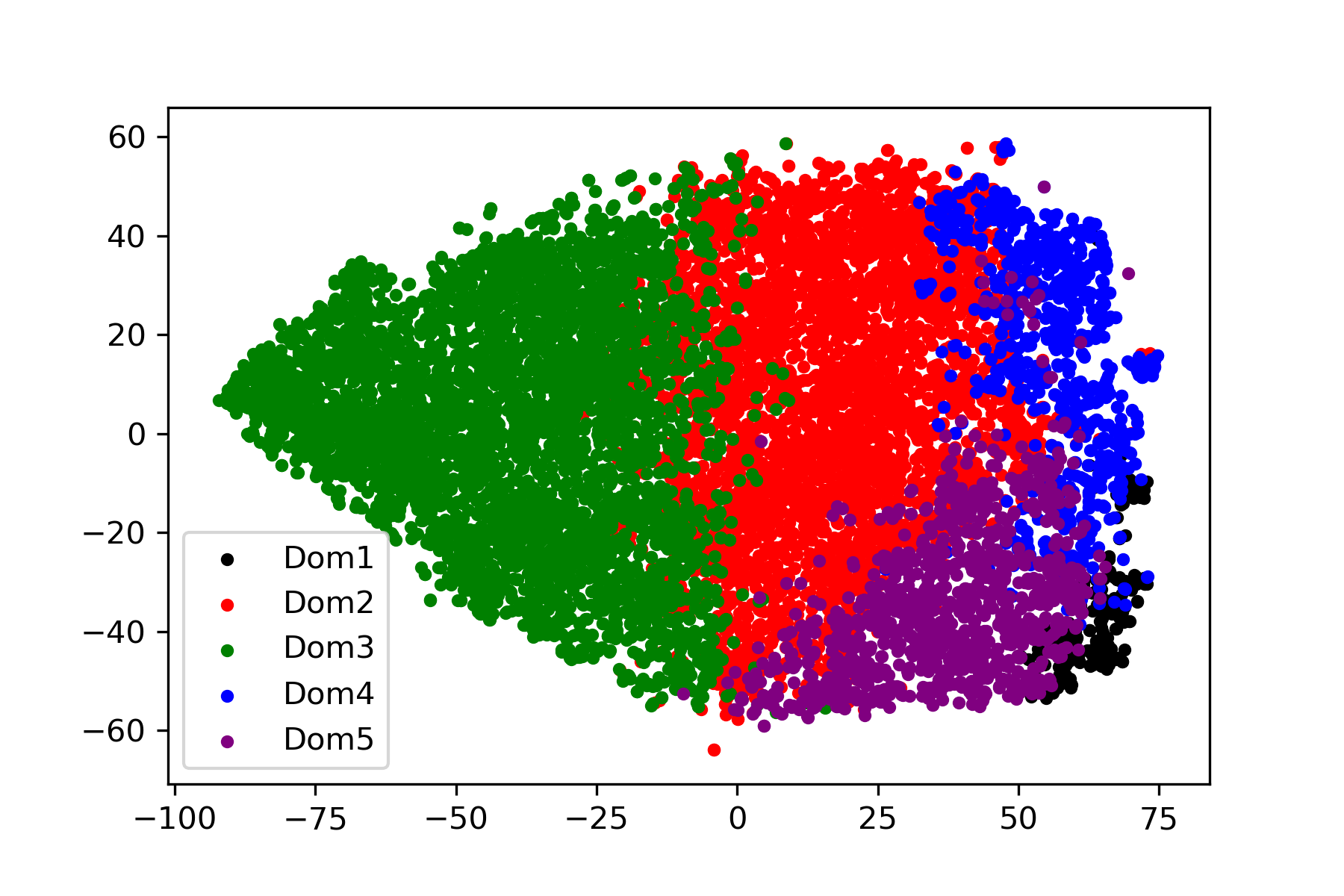}}
     \hfill
     \subfloat[]{         
        \includegraphics[width=.49\linewidth]{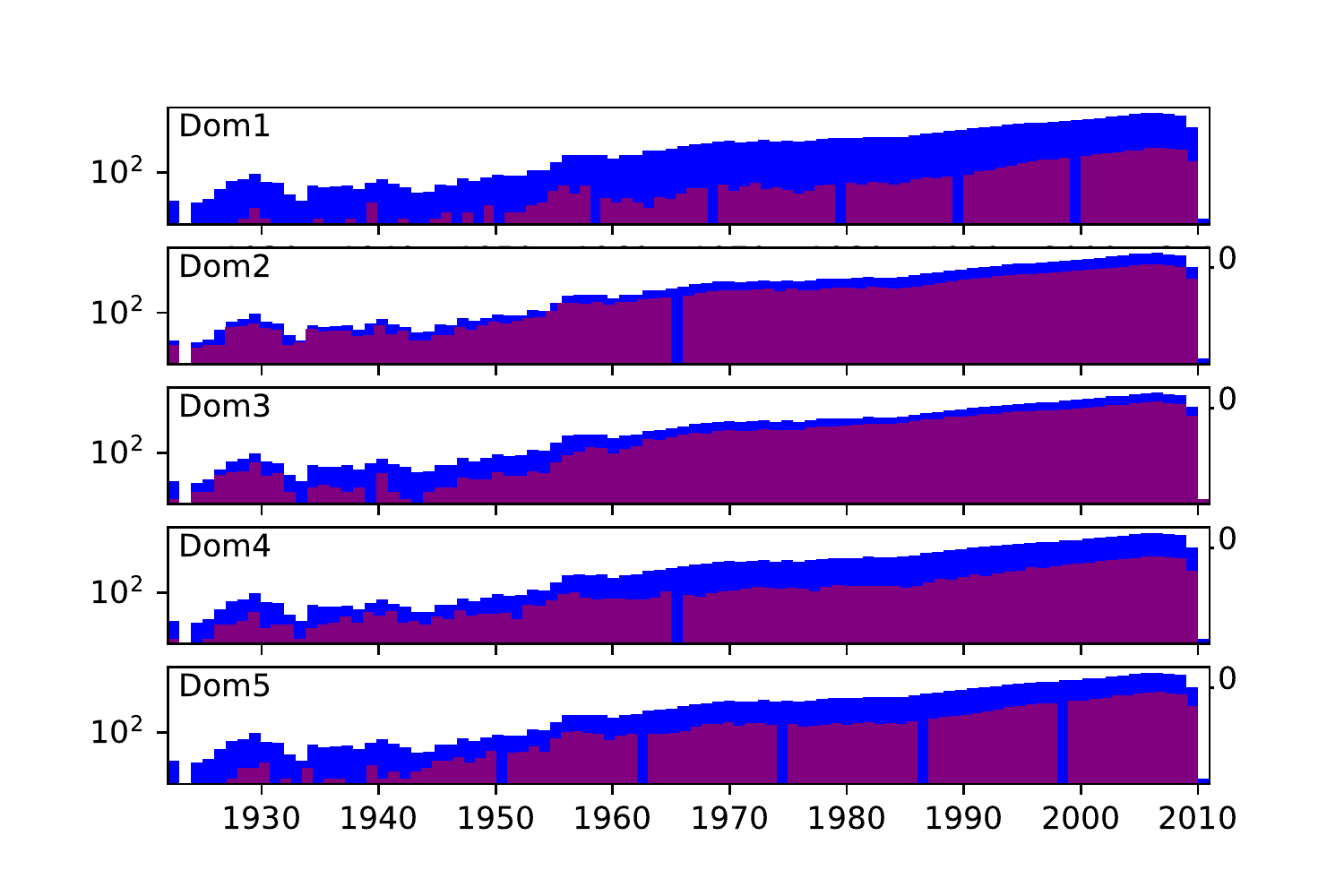}}
    \caption{The transformation of the YearPredictionMSD data into multiple domains. (Left) $t$-SNE visualization of the different domains discoverd in the YearPredictionMSD data. (Right) The histograms of the release year in each domain compared to that of the whole dataset (seen in the background in blue). The log scale is applied on the year in the y-axis.}
    \label{fig:YearPredictionMSD}
\end{figure}    

\subsubsection{The relative location of CT slices on axial axis dataset}
\label{sec:relative_CT}
The task beyond the relative location of CT \cite{graf20112d} is to predict the location of an image on the axial axis based on two histograms in polar space. It includes a set of 53500 CT images for 74 different patients.
We obtain the version hosted at the UCI repository \cite{Dua:2019}. As for the YearPredictionMSD data, we create a multi-source problem by assigning each patient randomly to a group; thereafter, we consider each group as a domain. The resulting domains are $\{$Dom$1,\dots,$Dom$5\}$. Figure \ref{fig:CT} presents the resulting five domains shown in a t-SNE embedding; the histograms show how the distribution of the target attribute differs between the five domains (groups of patients).

The quantitative results on relative CT dataset are summarized in Table~\ref{tab:CT}. We use the same baselines and state-of-the-art methods as described in the manuscript, except for DANN-1S, which fails to generalize and performs very poorly on all domains; hence, we omit it. The results confirm that our MDD always commits the slightest mean absolute error for all target domains.

\begin{figure}[h]
     \subfloat[]{
        \includegraphics[width=.49\linewidth]{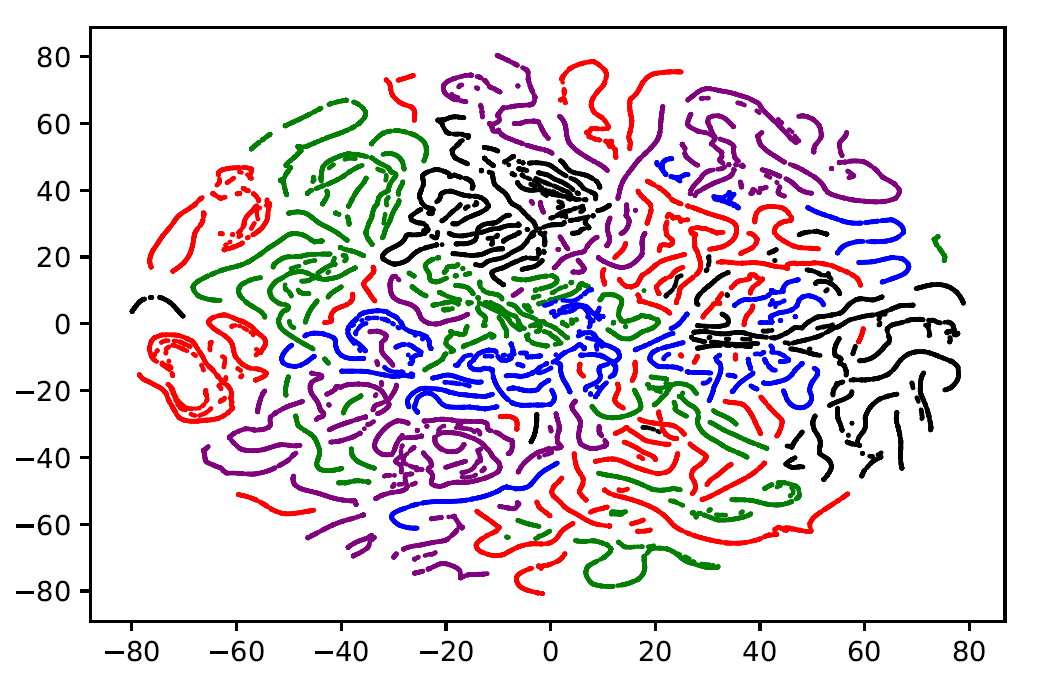}}
     \hfill
     \subfloat[]{         
        \includegraphics[width=.49\linewidth]{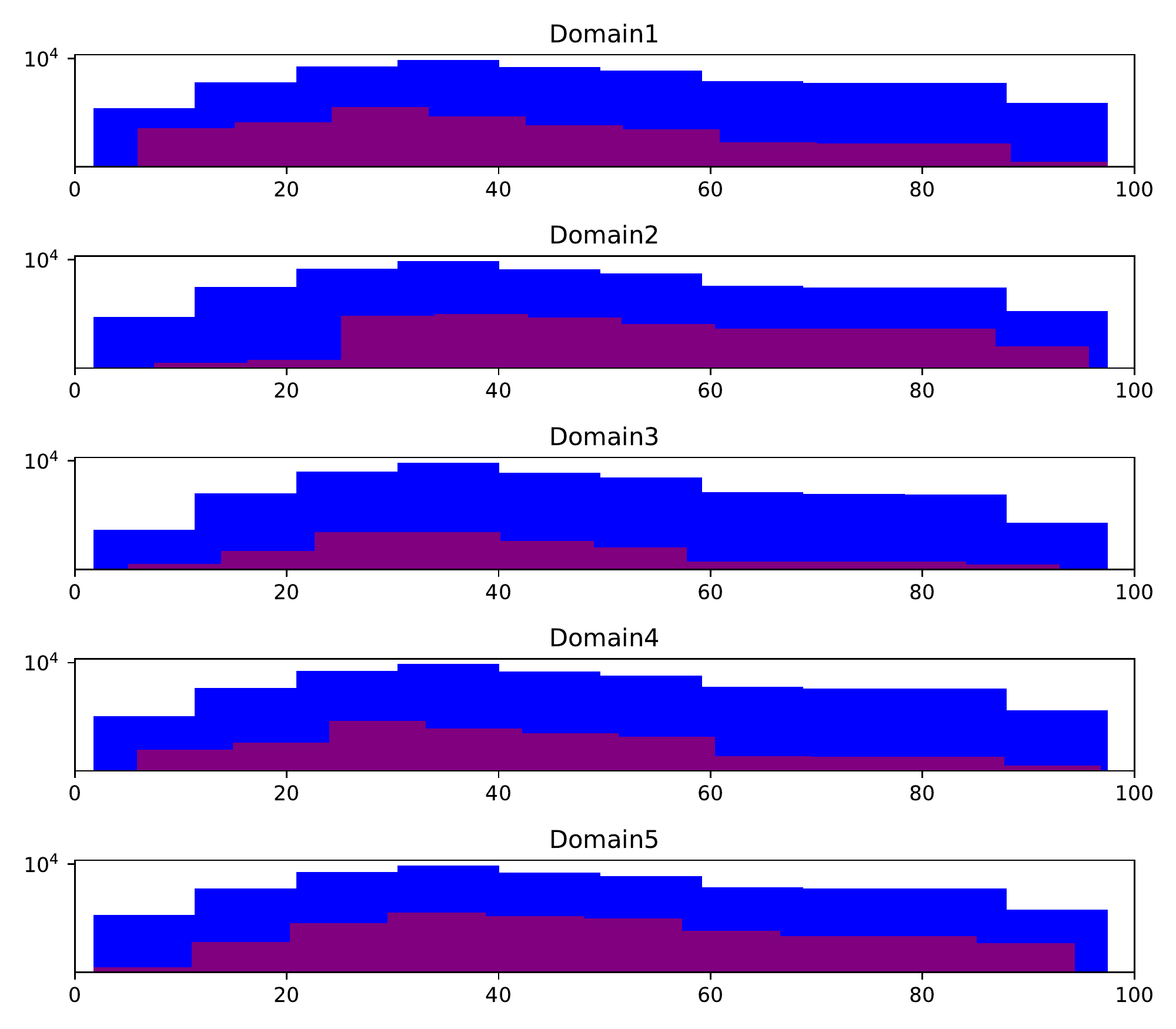}}
    \caption{The transformation of the CT data into multiple domains. (Left) $t$-SNE visualization of the different domains discovered in the relative CT data. (Right) The histograms of the 
    locations in the axial axis in each domain compared to that of the whole dataset (seen in the background in blue).}
    \label{fig:CT}
\end{figure}

\begin{table}[h]
	\caption{Performance comparison in terms of mean absolute error (MAE) over five iterations on relative CT data (with standard error in brackets).}\label{tab:CT}
	  \centering
	\begin{tabular}{p{0.04\textwidth}ccc|cc|c}
		\toprule
		&AHD&AHD-&DARN&\multicolumn{2}{c|}{MDAN}&MDD\\
		&-1S&MSDA&&-Max&-Dyn& \\
		\midrule
		Dom1&5.92 \small{(0.07)} & 5.95 \small{(0.17)} & 19.83 \small{(0.08} & 5.47 \small{(0.07)}&4.89 \small{(0.07)}&\textbf{4.45} \small{(0.09)}\\
		Dom2&5.52 \small{(0.12)} & 5.01 \small{(0.13)} & 18.36 \small{(0.11} & 5.13 \small{(0.06)}&4.32 \small{(0.04)}&\textbf{4.27} \small{(0.04)}\\
		Dom3&6.33 \small{(0.08)} & 5.39 \small{(0.10)} & 19.49 \small{(0.07} & 5.20 \small{(0.13)}&4.55 \small{(0.10)}&\textbf{4.46} \small{(0.08)}\\
		Dom4&6.03 \small{(0.12)} & 5.69 \small{(0.14)} & 18.81 \small{(0.06} & 5.32 \small{(0.05)}&4.59 \small{(0.06)}&\textbf{4.35} \small{(0.07)}\\
		Dom5&6.02 \small{(0.14)} & 5.60 \small{(0.17)} & 17.89 \small{(0.02} & 5.45 \small{(0.06)}&4.98 \small{(0.07)}&\textbf{4.4} \small{(0.05)}\\
		\hline
	\end{tabular}
\end{table}

\subsubsection{Learned Weights by MDD}

In fact, the interpretability of the learned weights is hard to justify for the real-world data in which the ground truth relations between tasks are not available. This is also the motivation why we use synthetic data (in which the ground truth on the strength of domain relatedness is known) to judge if our MDD can learn meaningful weights. The results in subsection ``Visualizing Domain Importance in Synthetic Data" suggest that the weights learned by our MDD are indeed more interpretable than that learned by DARN and AHD-MSDA.

We also plot the weights in each training epoch. Fig.~\ref{fig:weights_epochs} shows that DARN found weights oscillate from zero to one back and forth until reducing the altitude of the jumps around the 30th epoch. This explains the bad generalization observed on the vehicle counting problem as discussed above.

MDD, on the other hand, has a smooth development of the found weights, as seen on test domains 0, 2, and 5 (subfigures (A), (C), and (F) ). Another important observation that can be seen is that sometimes weights get stuck in local minima, which is eventually overcome in the following epochs. See, for example, Subfigure (D), where domain 1 gains a significant weight before being discovered as irrelevant and then gets down-weighted.

On the other hand, it is worth noting that, the relatively smooth evolution results of DARN in the classification case (Fig. 5 in~\cite{wen2020domain}) are actually generated by ``exponential moving averages with a decay rate of $0.95$".

\begin{figure}
	\centering	
	\includegraphics[width=0.85\textwidth]{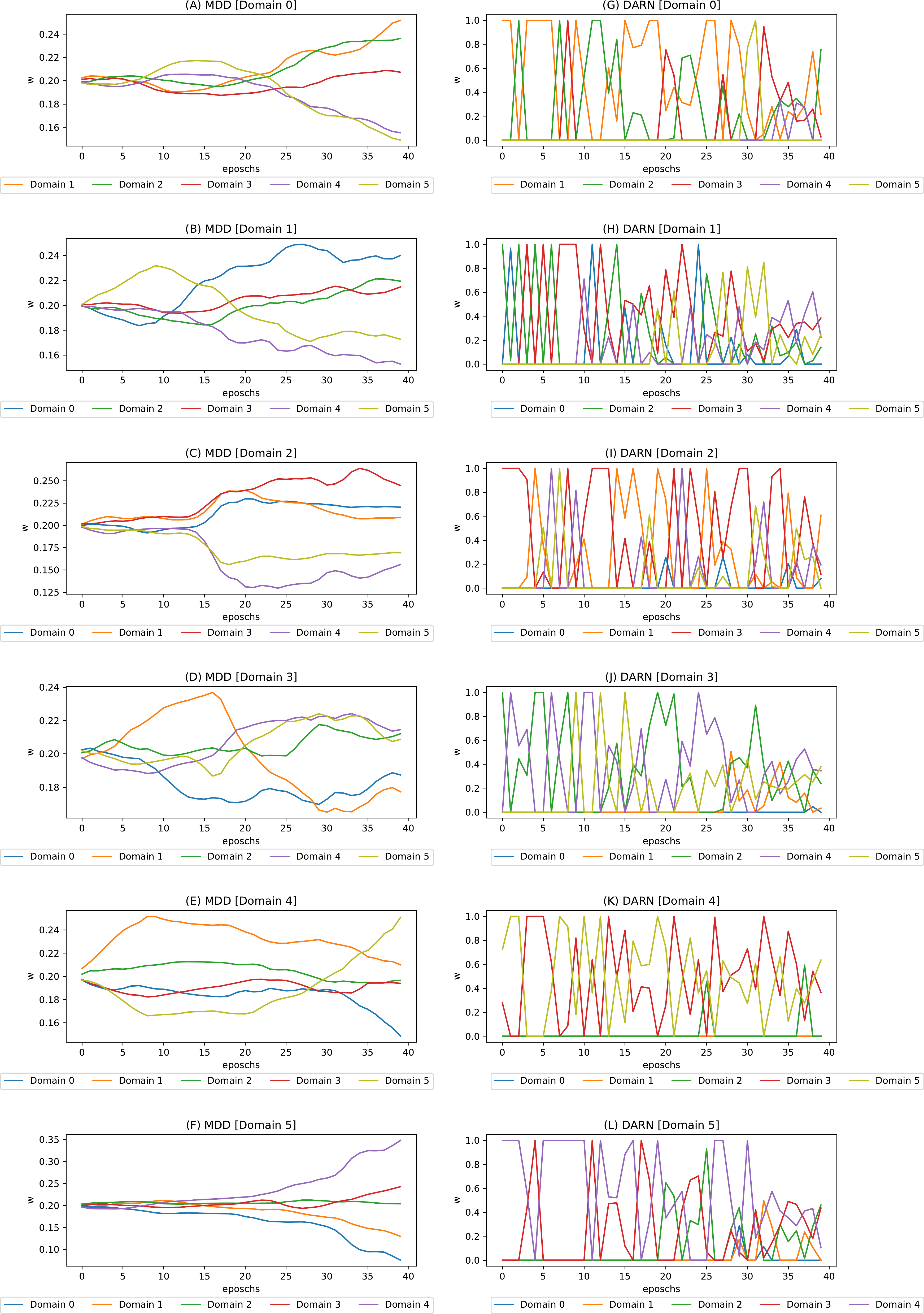}
	\caption{The learned source weight by MDD and DARN on the synthetic dataset in different training epochs.}
	\label{fig:weights_epochs}
\end{figure}

Here, we additionally plot the weights that our MDD learned in each adaptation scenario on Amazon Review dataset. Not surprisingly, we observed similar observations as on the synthetic data. For example, when the target is ``computer$\&$video-games”, our MDD selects ``electronics” as the source with the richest information (see Fig.~\ref{fig:alphas}). These two domains have more semantic similarity, because they have overlapping products.

\begin{figure}
	\centering	
	\includegraphics[width=0.3\textwidth]{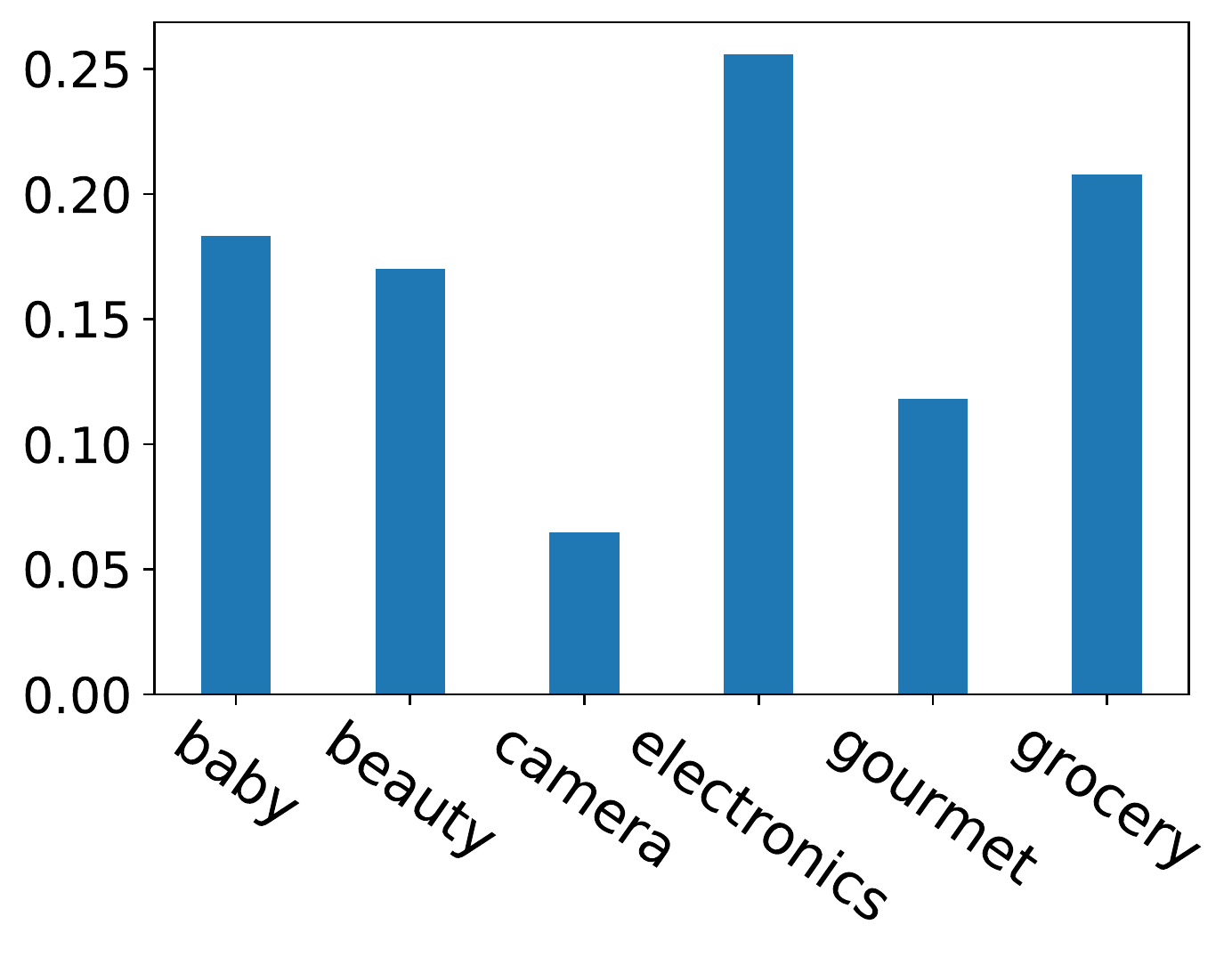}
	\caption{The learned source weight by MDD when target is ``computer\& games".}
	\label{fig:alphas}
\end{figure}

\subsection{Additional Results and Information for Evaluating the Continual Learning by Representation Similarity Penalty}
\label{Additional information for Evaluating CL}
In our experiments, we used the following continual learning libraries:
\begin{itemize}
    \item MER: Apache License, Version 2.0. \url{https://github.com/mattriemer/MER/blob/master/LICENSE}.\\
    This repository also offers the implementation of EWC and GEM. We used GEM's implementation as a basis to implement AGEM.
    \item REWC: MIT License. \url{https://github.com/xialeiliu/RotateNetworks}
\end{itemize}

\subsubsection{Hyperparameter Search}	
To ensure a fair comparison, we start with a grid-based hyperparameter search for each of the methods on each of the datasets using a sample of 5 tasks and 300 samples per task. The found parameters are reported in the following:\\ 	\begin{itemize}
	\item EWC found hyperparameters:
	\begin{itemize}
	\item learning rate: $lr \in \{$0.001(Omni), 0.003 (notmnistP), 0.01 (mnistR, mnistP, fashionP), 0.03, 0.1, 0.3, 1.0 $\}$ 
	\item regularization: $\lambda \in \{$1 (notmnistP), 3 (mnistR), 10 (Omni), 30, 100 (mnistP, fashionP), 300, 1000, 3000, 10000, 30000$\}$
	\end{itemize}
	\item R-EWC found hyperparameters:
	\begin{itemize}
	\item learning rate: $lr \in \{$0.001 (fashionP, mnistP, mnistR, notmnistP), 0.003, 0.01, 0.03, 0.1, 0.3, 1.0 $\}$ 
	\item regularization: $\lambda \in \{$1, 3, 10, 30 (mnistP, mnistR), 100, 300, 1000 (fashionP), 3000, 10000 (notmnistP), 30000$\}$
	\end{itemize}
	
	\item Meta-Experience Replay found hyperparameters:
    \begin{itemize}
	\item learning rate: $lr \in \{$0.001, 0.003, 0.005 (Omni), 0.01, 0.03, 0.1 (fashionP, mnistP, mnistR, notmnistP) $\}$ 
	\item across batch meta-learning rate: $\gamma = 1$
	\item within batch meta-learning rate: $\beta \in \{$0.01 (fashionP, mnistP, mnistR), 0.03 (notmnistP), 0.1, 0.3, 1.0 (Omni)$\}$
	\end{itemize}
	
	\item AGEM found hyperparameters:
	\begin{itemize}
	\item learning rate: $lr \in \{$0.001, 0.003, 0.005 (Omni), 0.01 (notmnistP,mnistR, mnistP, fashionP), 0.03, 0.1, $\}$ 
	\item memory strength: $ms \in 
		\{$0.0 \text{(notmnistP)}, 0.1, 0.5  \text{(mnistR, mnistP, fashionP, Omni)}, 1.0$\}$
	\end{itemize}
	\end{itemize}
	
    Without any further tuning, we adopt the same found parameters to our proposed modification, except for the memory strength, in RSP, that we force to be less than $10$.
    
\subsubsection{Sensitivity Analysis on the Number of Groups}
In this experiment, we study the sensitivity on the number of groups used by RSP. The analysis considers different numbers of groups, i.e., $K_d \in \{5, 10, 15, 20\}$ for all $d$. 
Table~\ref{tb:EWCDiffNumGroups} shows that RSP is insensitive to the number of groups. This can be inferred by the 
very small slope of the performance curve when increasing the number of groups $K_d$.

\begin{table}
		\centering
		\caption{Retained accuracy for RSP when a different number of groups is used. Five iterations are used. The numbers in parentheses are the standard error.}
		\begin{tabularx}{\columnwidth}{p{0.15\columnwidth}| p{0.16\columnwidth} p{0.16\columnwidth} p{0.16\columnwidth} p{0.16\columnwidth}}
		\hline		
			\text{Data} & \text{5} &  \text{10} &  \text{15} &  \text{20}\\
			\hline		
			\multirow{1}{*}{notmnistP}&71.0\small{(0.7)}&71.27\small{(0.7)}&71.01\small{(0.6)}&71.45\small{(0.7)}\\
			\multirow{1}{*}{fashionP}&63.16\small{(0.6)}&63.2\small{(0.5)}&63.62\small{(0.5)}&64.31\small{(0.6)}\\
			\multirow{1}{*}{mnistR}&61.43\small{(0.5)}&61.02\small{(0.5)}&61.63\small{(0.4)}&61.85\small{(0.3)}\\
			\multirow{1}{*}{mnistP}&71.91\small{(0.8)}&72.34\small{(0.5)}&72.08\small{(0.6)}&71.78\small{(0.6)}\\
			\hline
		\end{tabularx}
		\label{tb:EWCDiffNumGroups}
	\end{table}

\subsubsection{Experiment on Omniglot}	
We also explore the ability of RSP to overcome forgetting on the Omniglot dataset \cite{lake2011one}. We restrict the experiment on the first ten alphabets, and, unlike the online setting used in our previous experiments, we allow 500 epochs per task and a block size of $40$ samples. 

We follow \cite{vinyals2016matching,riemer_learning_2018} and use an architecture containing four blocks each of which contains a $3 \times 3$ convolution with 64 filters, a Relu activation and $2\times2$ max-pooling. The blocks are followed by two-fully connected layers and then multiple heads, one for each task. 
RSP operates by first applying grouping on each of the fully-connected layers, and then computing the parameter penalties based on the induced tasks' representation similarities by the groups they belong to, as explained earlier.
For EWC, we use the suggested parameters by \cite{riemer_learning_2018} ($lr=0.001$ and $ms=10$), and find 
$lr=0.005$ and $ms=0.005$) for RSP.
For MER, we set the $samples\_per\_block=10$ and reduce the number of epochs by $10$ for a fair comparison.
Table \ref{tb:omniglot} shows that RSP, despite the drop in performance, still presents better retained and learning accuracies.

In this experiment, we also tried to compare with REWC. Still, unlike the other methods, REWC requires a lot of engineering effort to introduce the rotation layers needed before and after each network's layer. This drawback makes it laborious to adopt REWC to new architectures. After a successful adaptation, REWC's performance was not competitive with the other methods, hence, we omitted its results.
\begin{table}[h]
	\centering
	\caption{Performance comparison between RSP and EWC on the Omniglot dataset. The numbers in parentheses are the standard errors (SE) of the means in the former row.}		
	\begin{tabularx}{0.9\columnwidth}{p{0.08\columnwidth} |rrr|rrr|rrr|rrr}
		\hline
		& \multicolumn{3}{c|}{AGEM} & \multicolumn{3}{c|}{MER} &  \multicolumn{3}{c|}{EWC} & \multicolumn{3}{c}{RSP} \\
		\text{} & \text{RA} &  \text{LA} &  \text{BT}& \text{RA} &  \text{LA} &  \text{BT}& \text{RA} &  \text{LA} &  \text{BT}& \text{RA} &  \text{LA} &  \text{BT} \\
		\hline
		\multirow{2}{*}{Omniglot}
&.095 &.236 &.142 &.072 &.109 &.037 &.054&.221&-.17&\textbf{.105}&\textbf{.434}&-.33\\
    &\small{(.004)} &\small{(.011)} &\small{(.01)}&\small{(.003)} &\small{(.003)} &\small{(.004)}&\small{(.005)}&\small{(.002)}&\small{(.005)}&\small{(.007)}&\small{(.006)}&\small{(.0)}\\  
		\hline
	\end{tabularx}
	\label{tb:omniglot}
\end{table}

\end{document}